\documentclass{article}
\usepackage{amsmath,amsfonts,amssymb,amsthm}
\usepackage{times}
\usepackage{fullpage}

\usepackage{algorithm} 
\usepackage{algorithmic}  
\usepackage{url}
\usepackage{graphicx}
\usepackage{caption} 
\usepackage{subcaption}
\usepackage[title]{appendix}
\usepackage[linesnumbered,ruled,algo2e]{algorithm2e}
\usepackage{float}
\usepackage{framed}
\usepackage{enumerate}
\usepackage{color}
\usepackage[colorlinks=true, linkcolor=red, urlcolor=blue, citecolor=blue]{hyperref}
\usepackage{bbm}

\usepackage{natbib}
\usepackage{graphicx,float,wrapfig}

\usepackage{stackengine}
\newcommand\xrowht[2][0]{\addstackgap[.5\dimexpr#2\relax]{\vphantom{#1}}}

\makeatletter
\def\set@curr@file#1{\def\@curr@file{#1}} 
\makeatother
\usepackage[load-configurations=version-1]{siunitx} 
\newcommand{\D}{\mathcal{D}}
\newcommand{\pr}{\mathrm{Pr}}
\newcommand{\poly}{\mathrm{poly}}

\newcommand{\cvar}{C_{\mathrm{var}}}
\newcommand{\chyper}{C_{\mathrm{hyper}}}

\newcommand{\lridge}{\lambda_{\rm ridge}}

\newcommand{\todo}[1]{\textcolor{red}{\textbf{TODO: #1}}}
\newcommand{\R}{\mathbb{R}}
\newcommand{\E}{\mathbb{E}}

\newtheorem{theorem}{Theorem}
\newtheorem{lemma}{Lemma}
\theoremstyle{definition}
\newtheorem{assumption}{Assumption}

\title{An Exponential Lower Bound for Linearly-Realizable MDPs\\ with Constant Suboptimality Gap}
\author{Yuanhao Wang\thanks{Princeton University. Email: \texttt{yuanhao@princeton.edu}} \and Ruosong Wang\thanks{Carnegie Mellon University. Email: \texttt{ruosongw@andrew.cmu.edu}} \and Sham M. Kakade\thanks{University of Washington and Microsoft Research. Email: \texttt{sham@cs.washington.edu}}}

\begin{document}

\maketitle
\begin{abstract}%
  A fundamental question in the theory of reinforcement learning is:
  suppose the optimal $Q$-function lies in the linear span of a given
  $d$ dimensional feature mapping, is sample-efficient reinforcement
  learning (RL) possible? The recent and remarkable result of Weisz et
  al. (2020) resolved this question in the negative, providing an
  exponential (in $d$) sample size lower bound, which holds even if
  the agent has access to a generative model of the environment. One
  may hope that this information theoretic barrier for RL can be
  circumvented by further supposing an even more favorable assumption: 
  there exists a \emph{constant suboptimality gap} between the optimal
  $Q$-value of the best action and that of the second-best action (for
  all states). 
  The hope is that having a large suboptimality gap would permit
  easier identification of optimal actions themselves, thus making the
  problem tractable; indeed, provided the agent has access
  to a generative model, sample-efficient RL is in fact possible
  with the addition of this more favorable assumption.



  This work focuses on this question in the standard online reinforcement learning setting,
  where our main result resolves this question in the negative: our
  hardness result shows that an exponential sample complexity lower
  bound still holds even if a constant suboptimality gap is assumed in
  addition to having a linearly realizable optimal $Q$-function.
  Perhaps surprisingly, this implies an exponential separation between
  the online RL setting and the generative model
  setting. Complementing our negative hardness result, we give two
  positive results showing that provably sample-efficient RL is
  possible either under an additional low-variance assumption or under
  a novel hypercontractivity assumption (both implicitly place
  stronger conditions on the underlying dynamics model).
\end{abstract}

\section{Introduction}

There has been substantial recent theoretical interest in understanding
the means by which we can avoid the curse of dimensionality and obtain
 sample-efficient reinforcement learning (RL) methods~\citep{wen2017efficient, du2019provably,
  du2019good, wang2019optimism, yang2019sample, lattimore2020learning,
  yang2020reinforcement, jin2020provably, cai2020provably,
  zanette2020learning, weisz2020exponential, du2020agnostic,
  zhou2020provably, zhou2020nearly, modi2020sample, jia2020model,
  ayoub2020model}.  Here, the extant body of literature largely
focuses on sufficient conditions for efficient reinforcement
learning. Our understanding of what are the necessary
conditions for efficient reinforcement learning is far more limited.
With regards to the latter, arguably, the most
natural assumption is linear realizability: we assume that the optimal
$Q$-function lies in the linear span of a given feature map.  The goal
is to the obtain polynomial sample complexity under this linear realizability assumption alone.

This ``linear $Q^*$ problem'' was a major
open problem (see~\citet{du2019good} for discussion), and
a recent hardness result by~\citet{weisz2020exponential} provided a
negative answer.
In particular, the result showed that even with access to a generative model,
any algorithm requires an exponential number of samples (in the
dimension $d$ of the feature mapping) to find a
near-optimal policy, provided the action space has exponential size. 

With this question resolved, one may naturally ask what is the source of
hardness for the construction in~\citet{weisz2020exponential} and if  there
are additional assumptions 
that can serve to bypass the underlying source of this hardness.  
Here, arguably, it is most natural to further examine the
suboptimality gap in the problem, which
is the gap between the optimal $Q$-value of the best action and
that of the second-best action; the construction
in~\citet{weisz2020exponential} does in fact fundamentally rely on
having an exponentially small gap. Instead, if we assume the gap is lower bounded by a constant for all
states, we may hope that the problem is substantially easier since with a finite number
of samples (appropriately obtained), we can identify the optimal policy itself
(i.e. gap assumptions allows us to translate value-based accuracy to 
identification of the optimal policy itself).
In fact, this intuition is correct in the following sense: with a
generative model, it is not difficult to see that polynomial sample
complexity is possible under the linear realizability assumption plus
the suboptimality gap assumption, since the suboptimality gap
assumption allows us to easily identify an optimal action for all
states, thus making the problem tractable (see Section C
in~\citet{du2019good} for a formal argument).

More generally, the suboptimality gap assumption is widely discussed in the
bandit literature~\citep{dani2008stochastic, audibert2010best,
  abbasi2011improved} and the reinforcement learning
literature~\citep{simchowitz2019non, yang2020q} to obtain fine-grained
sample complexity upper bounds.  More specifically, under the realizability
assumption and the suboptimality gap assumption, it has been shown
that polynomial sample complexity is possible if the transition is
nearly deterministic~\citep{du2020agnostic} (also see~\citet{wen2017efficient}).  However, it remains unclear whether the
suboptimality gap assumption is sufficient to bypass the hardness
result in~\citet{weisz2020exponential}, or the same exponential lower
bound still holds even under the suboptimality gap assumption, when
the transition could be stochastic and the generative model is
unavailable.  For the construction in~\citet{weisz2020exponential}, at
the final stage, the gap between the value of the optimal action and
its non-optimal counterparts will be exponentially small, and
therefore the same construction does not imply an exponential sample
complexity lower bound under the suboptimality gap assumption.

\begin{table}[tb]
    \centering
    \begin{tabular}{|c|c|c|}
\hline
    \textbf{Minimum Gap?}    & \textbf{Generative Model} & \textbf{Online RL}\\ 
    \hline\xrowht[()]{12pt}
No & Exponential~\citep{weisz2020exponential}  & Exponential~\citep{weisz2020exponential} \\
 \hline\xrowht[()]{12pt}
    Yes & Polynomial~\citep{du2019good} & \textbf{Exponential (This work, Theorem~\ref{thm:lower})}\\
    \hline
    \end{tabular}
    \caption{Known sample complexity results for RL with linear function approximation under realizability. ``Exponential'' refers to exponential lower bound (in the dimension or horizon), while ``polynomial'' refers to a polynomial upper bound.}
    \label{tab:my_label}
\end{table}

\paragraph{Our contributions.}
Following~\citet{weisz2020exponential}, in this paper, we prove a new hardness result for online RL with linear realizability. 
In particular, we show that in the online RL setting where a generative model is unavailable (and is therefore a weaker setting) with exponential-sized action space, the exponential sample complexity lower bound still holds even under the suboptimality gap assumption.
Complementing our hardness result, we show that under the realizability assumption and the suboptimality gap assumption, our hardness result can be bypassed if one further assumes the low variance assumption in~\citet{du2019provably}~\footnote{We note that the sample complexity of the algorithm in~\citet{du2019provably} has at least linear dependency on the number of actions, which is not sufficient for bypassing our hardness results which assumes an exponential-sized action space. }, or a hypercontractivity assumption. Hypercontractive distributions include Gaussian distributions (with arbitrary covariance matrices), uniform distributions over hypercubes and strongly log-concave distributions~\citep{kothari2017better}. This condition has been shown powerful for outlier-robust linear regression~\citep{kothari2017outlier}, but has not yet been introduced for reinforcement learning with linear function approximation.

Our results have several interesting implications, which we discuss in detail in Section~\ref{sec:discussion}. Most notably, our results imply an \emph{exponential} separation between the standard reinforcement learning setting and the generative model setting.
Moreover, our construction enjoys greater simplicity, making it more suitable to be generalized for other RL problems or to be presented for pedagogical purposes.

\paragraph{Organization.} This paper is organized as follows. In Section~\ref{sec:related}, we review related works in literature. In Section~\ref{sec:pre}, we introduce necessary notations, definitions and assumptions. In Section~\ref{sec:lower}, we present our hardness result. In Section~\ref{sec:upper}, we present the upper bound. Finally we discuss implications of our results and open problems in Section~\ref{sec:discussion}.

\section{Related work}\label{sec:related}

\paragraph{Previous hardness results.}
Existing exponential lower bounds in RL~\citep{krishnamurthy2016pac, chen2019information} usually construct unstructured MDPs with an exponentially large state space.
\citet{du2019good} proved that under the approximate version of the realizability assumption, i.e., the optimal $Q$-function lies in the linear span of a given feature mapping approximately, any algorithm requires an exponential number of samples to find a near-optimal policy. 
The main idea in~\citet{du2019good} is to use the Johnson-Lindenstrauss lemma~\citep{johnson1984extensions} to construct a large set of near-orthogonal feature vectors. 
Such idea is later generalized to other settings, including those in~\citet{wang2020reward, kumar2020discor, van2019comments, lattimore2020learning}.
Whether the exponential lower bound still holds under the exact version of the realizability assumption is left as an open problem in~\citet{du2019good}.

The above open problem is recently solved by~\citet{weisz2020exponential}. They show that under the exact version of the realizability assumption, any algorithm requires an exponential number of samples to find a near-optimal policy assuming an exponential-sized action space. 
The construction in~\citet{weisz2020exponential} also uses the Johnson-Lindenstrauss lemma to construct a large set of near-orthogonal feature vectors, with additional subtleties to ensure exact realizability. 

Very recently, under the exact realizability assumption, strong lower bounds are proved in the offline setting~\citep{wang2020statistical, zanette2020exponential}. These work focus on the offline RL setting, where a fixed data distribution with sufficient coverage is given and the agent cannot interact with the environment in an online manner. Instead, we focus on the online RL setting in this paper. 

\paragraph{Existing upper bounds.}
For RL with linear function approximation, most existing upper bounds require representation conditions stronger than realizability.
For example, the algorithms in~\citet{yang2019sample, yang2020reinforcement, jin2020provably, cai2020provably, zhou2020provably, zhou2020nearly, modi2020sample, jia2020model, ayoub2020model} assume that the transition model lies in the linear span of a given feature mapping, and the algorithms in~\citet{wang2019optimism, lattimore2020learning, zanette2020learning} assume completeness properties of the given feature mapping. 
In the remaining part of this section, we mostly focus on previous upper bounds that requires only realizability as the representation condition. 

For deterministic systems, under the realizability assumption, \citet{wen2017efficient} provide an algorithm that achieves polynomial sample complexity.
Later, under the realizability assumption and the suboptimality gap assumption, polynomial sample complexity upper bounds are shown if the transition is deterministic~\citep{du2020agnostic}, a generative model is available~\citep{du2019good}, or a low-variance condition holds~\citep{du2019provably}.
Compared to the original algorithm in~\citet{du2019provably}, our modified algorithm in Section~\ref{sec:upper} works under a similar low-variance condition. However, the sample complexity in~\citet{du2019provably} has at least linear dependency on the number of actions, whereas our sample complexity in Section~\ref{sec:upper} has no dependency on the size of the action space. 
Finally,~\citet{shariff2020efficient} obtains a polynomial upper bound under the realizability assumption when the features for all state-action pairs are inside the convex hull of a polynomial-sized coreset and the generative model is available to the agent. 

\section{Preliminaries}\label{sec:pre}
Throughout a paper, we use $[N]$ to denote the set $\{1, 2, \ldots, N\}$.
For a set $S$,we use $\Delta_S$ to denote the probability simplex. 
\subsection{Markov decision process (MDP) and reinforcement learning}
An MDP is specified by $(\mathcal{S},\mathcal{A},H,P,\{R_h\}_{h\in [H]})$, where $\mathcal{S}$ is the state space, $\mathcal{A}$ is the action space with $|\mathcal{A}|=A$, $H\in\mathbb{Z}^{+}$ is the planning horizon, $P:\mathcal{S}\times\mathcal{A}\to\Delta_{\mathcal{S}}$ is the transition function and $R_h:\mathcal{S}\times\mathcal{A}\to\Delta_{\R}$ is the reward distribution. Throughout the paper, we occasionally abuse notation and use a scalar $a$ to denote the single-point distribution at $a$.

A (stochastic) policy takes the form $\pi=\{\pi_h\}_{h\in [H]}$, where each $\pi_h:\mathcal{S}\to\Delta_{\mathcal{A}}$ assigns a distribution over actions for each state. We assume that the initial state is drawn from a fixed distribution, i.e. $s_1\sim \mu$. Starting from the initial state, a policy $\pi$ induces a random trajectory $s_1,a_1,r_1,\cdots,s_H,a_H,r_H$ via the process $a_h\sim \pi_h(\cdot)$, $r_h\sim R(\cdot|s_h,a_h)$ and $s_{h+1}\sim P(\cdot|s_h,a_h)$. For a policy $\pi$, denote the distribution of $s_h$ in its induced trajectory by $\D^\pi_h$. 

Given a policy $\pi$, the $Q$-function (action-value function) is defined as $$Q_h^\pi(s,a):=\E\left[\sum_{h'=h}^H r_{h'}|s_h=s,a_h=a,\pi\right],$$ while $V_h^\pi(s):=\E_{a\sim \pi_h(s)}[Q_h^{\pi}(s,a)]$. We denote the optimal policy by $\pi^*$, and the associated optimal $Q$-function and value function by $Q^*$ and $V^*$ respectively. Note that $Q^*$ and $V^*$ can also be defined via the Bellman optimality equation\footnote{We additionally define $V_{H+1}(s)=0$ for all $s\in\mathcal{S}$.}:
\begin{align*}
    V_h^*(s) &= \max_{a\in\mathcal{A}}Q_h^*(s,a),\\
    Q_h^*(s,a)&= \E\left[R_h(s,a)+V^*_{h + 1}(s_{h+1})|s_h=s,a_h=a\right].
\end{align*}

\paragraph{The online RL setting.}
In this paper, we aim to prove lower bound and upper bound in the online RL setting.
In this setting, in each episode, the agent interacts with the unknown environment using a policy and observes rewards and the next states.
We remark that the hardness result by~\citet{weisz2020exponential} operates in the setting where a generative model is available to the agent so that the agent can transit to any state.
Also, it is known that with a generative model, under the linear realizability assumption plus the suboptimality gap assumption,
one can find a near-optimal policy with polynomial number of samples  (see Section C in~\citet{du2019good} for a formal argument).

\subsection{Linear $Q^\star$  function approximation}
When the state space is large or infinite, structures on the state space are necessary for efficient reinforcement learning. In this work we consider linear function approximation. Specifically, there exists a feature map $\phi:\mathcal{S}\times\mathcal{A}\to\R^d$, and we will use linear functions of $\phi$ to represent $Q$-functions of the MDP. To ensure that such function approximation is viable, we assume that the optimal $Q$-function is realizable. 
\begin{assumption}[Realizability]
\label{assumption:realizability}
For all $h\in [H]$, there exists $\theta^*_h\in\R^d$ such that for all $(s, a)\in \mathcal{S} \times \mathcal{A}$, $Q^*_h(s,a)=\phi(s,a)^{\top}\theta^*_h$.
\end{assumption}
This assumption is widely used in existing reinforcement learning and contextual bandit literature~\citep{du2019provably,foster2020beyond}. However, even for linear function approximation, realizability alone is not sufficient for sample-efficient reinforcement learning ~\citep{weisz2020exponential}.

In this work, we also impose the regularity condition that $\Vert \theta^*_h\Vert_2=O(1)$ and $\Vert\phi(s,a)\Vert_2=O(1)$, which can always be achieved via rescaling.

Another assumption that we will use is that the minimum suboptimality gap is lower bounded.
\begin{assumption}[Minimum Gap]
\label{assumption:gap}
For any state $s\in \mathcal{S}$, $a\in\mathcal{A}$, the suboptimality gap is defined as $\Delta_h(s,a):=V_h^*(s)-Q_h^*(s,a)$. We assume that $\min_{h\in [H],s\in\mathcal{S},a\in\mathcal{A}}\left\{\Delta_h(s,a):\Delta_h(s,a)>0\right\}\ge \Delta_{\min}$.
\end{assumption}
As mentioned in the introduction, this assumption is common in bandit and reinforcement learning literature. 

\section{Hard Instance with Constant Suboptimality Gap}
\label{sec:lower}
We now present our main hardness result.  

\begin{theorem}\label{thm:lower}
There exist universal constants $c_0, c_1 > 0 $ such
that the following statement holds. Consider any online RL algorithm $\mathbb{A}$ that takes the feature mapping $\phi : \mathcal{S} \times \mathcal{A} \to \mathbb{R}^d$ as input.
In the online RL setting, there exists an MDP with a feature mapping
$\phi$ satisfying Assumption~\ref{assumption:realizability} and
Assumption~\ref{assumption:gap} with $\Delta_{\min}\geq c_0$, 
such that $\mathbb{A}$ requires $2^{c_1 \cdot \min\{ d,H\} }$ samples to find a policy $\pi$ with
\begin{equation*}
   \E_{s_1\sim\mu}V^{\pi}(s_1)\ge \E_{s_1\sim\mu}V^*(s_1)-0.05
\end{equation*}
with probability greater than $0.1$ (or $\mathbb{A}$ fails to return such a policy
$\pi$ with probability greater than $0.1$).
\end{theorem}

\begin{figure}
    \centering
    \includegraphics[width=0.4\textwidth]{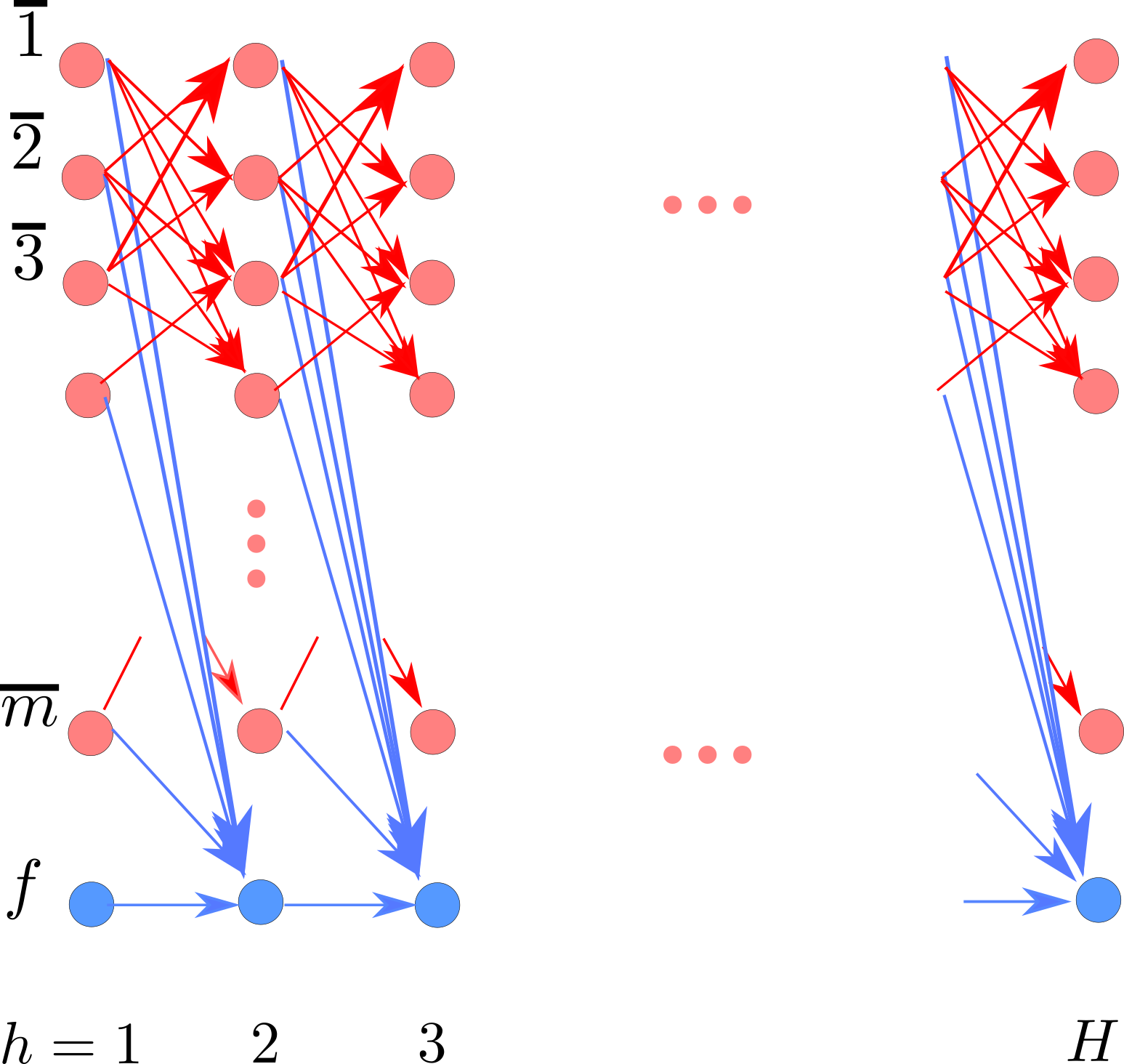}
    \caption{\textbf{The Leaking Complete Graph Construction.} Illustration of a hard MDP. There are $m+1$ states in the MDP, where $f$ is an absorbing terminal state. Starting from any non-terminal state $\overline{a}$, regardless of the action, there is at least $\gamma=1/6$ probability that the next state will be $f$.}
    \label{fig:mdp}
\end{figure}

The remainder of this section provides the construction of a hard
family of MDPs where $Q^*$ is linearly realizable and has constant
suboptimality gap and where it takes exponential samples to learn a
near-optimal policy.  Each of these hard MDPs can roughly
be seen as a ``leaking complete graph'' (see
Fig.~\ref{fig:mdp}). Information about the optimal policy can only be
gained by: (1) taking the optimal action; (2) reaching a non-terminal
state at level $H$. We will show that when there are exponentially
many actions, either events happen with negligible probability unless
exponentially many trajectories are played.

\subsection{Construction of the MDP family}
In this subsection we describe the construction of the hard instance (the hard MDP family) in detail.
Let $m$ be an integer to be determined. The state space is $\{\bar{1},\cdots,\bar{m},f\}$. The special state $f$ is called the \emph{terminal state}. 
The action space is simply $[m]$.
Each MDP in this family is specified by an index $a^*\in[m]$ and denoted by $\mathcal{M}_{a^*}$. In other words, there are $m$ MDPs in this family.

In order to construct the MDP family, we first find a set of approximately orthogonal vectors by leveraging the Johnson-Lindenstrauss lemma~\citep{johnson1984extensions}.
\begin{lemma}[Johnson-Lindenstrauss]
\label{lemma:jl}
For any $\gamma>0$, if $m\le \exp(\frac{1}{8}\gamma^2d')$, there exists $m$ unit vectors $\{v_1,\cdots,v_m\}$ in $\R^{d'}$ such that $\forall i,j\in[m]$ such that $i\neq j$, $|\langle v_i,v_j\rangle|\le \gamma$.
\end{lemma}
We will set $\gamma=\frac{1}{6}$ and $m=\lfloor \exp(\frac{1}{8}\gamma^2d)\rfloor$. By Lemma~\ref{lemma:jl}, we can find such a set of $d$-dimensional unit vectors $\{v_1,\cdots,v_m\}$. For the clarity of presentation, we will use $v_i$ and $v(i)$ interchangeably. The construction of $\mathcal{M}_{a^*}$ is specified below.
Note that in our construction, the features, the rewards and the transitions are defined for all $\overline{a_1},a_2$ with $a_1,a_2\in [m]$ and $a_1 \neq a_2$.
In particular, our construction is properly defined even when $a_1 = a^*$. 

\paragraph{Features.} The feature map, which maps state-action pairs to $d + 1$ dimensional vectors, is defined as follows. 
\begin{align*}
     \phi(\overline{a_1},a_2) &:= \left(0,\left(\Big\langle v(a_1),v(a_2)\Big\rangle+2\gamma\right)\cdot v(a_2)\right), \tag*{($\forall a_1,a_2\in [m], a_1\neq a_2$)}\\
   \phi(\overline{a_1},a_1) &:= \left(\frac{3}{4}\gamma,\mathbf{0}\right), \tag*{($\forall a_1\in [m]$)} \\
   \phi(f,1) &= \left(0,\mathbf{0}\right),\\
   \phi(f,a) &:= \left(-1,\mathbf{0}\right).\tag*{($\forall a\neq 1$)}
\end{align*}
	Here $\mathbf{0}$ is the zero vector in $\mathbb{R}^d$. 
    Note that the feature map is independent of $a^*$ and is shared across the MDP family. 

\paragraph{Rewards.} For $1\le h<H$, the rewards are defined as
    \begin{align*}
        R_h(\overline{a_1},a^*) &:= \Big\langle v(a_1), v(a^*)\Big\rangle+2\gamma,\tag*{($a_1\neq a^*$)}\\
        R_h(\overline{a_1},a_2) &:= -2\gamma\left[\Big\langle v(a_1),v(a_2)\Big\rangle+2\gamma\right], \tag*{($a_2\neq a^*$, $a_2\neq a_1$)}\\
        R_h(\overline{a_1},a_1) &:= \frac{3}{4}\gamma,\tag*{($\forall a_1$)}\\
        R_h(f,1) &:= 0, \\
        R_h(f,a) &:= -1. \tag*{($a\neq 1$)}
    \end{align*}
    For $h=H$, $r_H(s,a):=\langle \phi(s,a),  (1,  v(a^*))\rangle$ for every state-action pair.

\paragraph{Transitions.} The initial state distribution $\mu$ is set as a uniform distribution over $\{\bar{1},\cdots,\bar{m}\}$. The transition probabilities are set as follows.
\begin{align*}
\Pr[f|\overline{a_1},a^*] &= 1,\\
\Pr[f|\overline{a_1},a_1] &= 1,\\
    \Pr[\cdot|\overline{a_1},a_2] &= \begin{cases}
    \overline{a_2}:\Big\langle v(a_1), v(a_2)\Big\rangle+2\gamma \\
    f:1-\Big\langle v(a_1), v(a_2)\Big\rangle-2\gamma
    \end{cases},\tag*{($a_2\neq a^*$, $a_2\neq a_1$)}\\
\Pr[f|f,\cdot] &= 1.
\end{align*}

After taking action $a_2$, the next state is either $\overline{a_2}$ or $f$. Thus this MDP looks roughly like a ``leaking complete graph'' (see Fig.~\ref{fig:mdp}): starting from state $\overline{a}$, it is possible to visit any other state (except for $\overline{a^*}$); however, there is always at least $1-3\gamma$ probability of going to the terminal state $f$. The transition probabilities are indeed valid, because
$$0<\gamma\le \Big\langle v(a_1), v(a_2)\Big\rangle+2\gamma\le 3\gamma<1.$$
We now verify that linear realizability, i.e. Assumption~\ref{assumption:realizability}, is satisfied.
\begin{lemma}[Linear realizability]\label{lem:realizability}
In the MDP $\mathcal{M}_{a^*}$, $\forall h\in [H]$, for any state-action pair $(s,a)$, $Q^*_h(s,a) = \langle \phi(s,a),\theta^*\rangle$ with $\theta^* = (1,  v(a^*))$. 
\end{lemma}
\begin{proof}
We first verify the statement for the terminal state $f$. Observe that at the terminal state $f$, the next state is always $f$ and the reward is either $0$ (if action $1$ is chosen) or $-1$ (if an action other than $1$ is chosen). 
Hence, we have 
\[
Q_h^*(f, a) = \begin{cases}
0 & a  = 1\\
-1 & a \neq 1
\end{cases}
\]
and
\[
V_h^*(f)=0.
\]
This implies $Q_h^*(f,\cdot)=\langle\phi(f,\cdot),(1,  v(a^*)) \rangle $.

We now verify realizability for other states via induction on $h=H,H-1,\cdots,1$. The induction hypothesis is that for all $a_1, a_2 \in [m]$, we have 
\begin{equation}\label{eq:qstar}
    Q_h^*(\overline{a_1}, {a_2}) = \begin{cases}
    	\left(\left\langle v(a_1),v(a_2)\right\rangle+2\gamma\right)\cdot \left\langle v(a_2), v(a^*)\right\rangle & a_1 \neq a_2\\
	\frac{3}{4}\gamma  & a_1 = a_2
   \end{cases}
   \end{equation}
and
\begin{equation}\label{eq:vstar}
V_h^*(\overline{a_1}) = \begin{cases}
    	\left\langle v(a_1), v(a^*) \right\rangle + 2\gamma & a_1 \neq a^* \\
	\frac{3}{4}\gamma  & a_1 = a^*
   \end{cases}.
   \end{equation}

Note that (\ref{eq:qstar}) implies that realizability is satisfied. In the remaining part of the proof we verify Eq. (\ref{eq:qstar}) and (\ref{eq:vstar}). 

When $h=H$, (\ref{eq:qstar}) holds by the definition of rewards. Next, note that for all $h \in [H]$, (\ref{eq:vstar}) follows from (\ref{eq:qstar}). This is because for all $a_1\neq a^*$, for all $a_2\notin \{ a_1, a^*\}$.
\begin{align*}
    Q_h^*(\overline{a_1},{a_2}) &= \left(\left\langle v(a_1),v(a_2)\right\rangle+2\gamma\right)\cdot \left\langle v(a_2), v(a^*)\right\rangle\le 3\gamma^2,
\end{align*}
Moreover, for all $a_1\neq a^*$,
\begin{align*}
    Q_h^*(\overline{a_1},{a_1}) &= \frac{3}{4}\gamma < \gamma.
\end{align*}
Furthermore, for all $a_1\neq a^*$,
\begin{align*}
    Q_h^*(\overline{a_1},{a^*}) = \left\langle v(a_1), v(a^*) \right\rangle + 2\gamma \ge \gamma > 3\gamma^2.
\end{align*}
In other words, (\ref{eq:qstar}) implies that $a^*$ is always the optimal action for all state $\overline{a_1}$ with $a_1 \neq a^*$. 
Now, for state $\overline{a^*}$, for all $a \neq a^*$, we have
\begin{align*}
    Q_h^*(\overline{a^*}, a) &= \left(\left\langle v(a^*),v(a)\right\rangle+2\gamma\right)\cdot \left\langle v(a^*), v(a)\right\rangle\le 3\gamma^2 < \frac{3}{4}\gamma = Q_h^*(\overline{a^*}, a^*).
\end{align*}
Hence, (\ref{eq:qstar}) implies that $a^*$ is always the optimal action for all states $\overline{a}$ with $a \in [m]$.

Thus, it remains to show that (\ref{eq:qstar}) holds for $h$ assuming (\ref{eq:vstar}) holds for $h+1$. 
Here we only consider the case that $a_2 \neq a_1$ and $a_2 \neq a^*$, since otherwise $\Pr[f|\overline{a_1},a_2] = 1$ and thus (\ref{eq:qstar}) holds by the definition of the rewards and the fact that $V_h^*(f) = 0$. When $a_2 \notin \{a_1, a^*\}$, we have
\begin{align*}
    Q_h^*(\overline{a_1},{a_2}) &= R_h(\overline{a_1},a_2) + \E_{s_{h+1}}\left[\left.V_{h+1}^*(s_{h+1})\right|\overline{a_1},a_2\right]\\
    &=-2\gamma\left[\Big\langle v(a_1),v(a_2)\Big\rangle+2\gamma\right] + \Pr[s_{h+1}=\overline{a_2}]\cdot V_{h+1}^*(a_2)+\Pr[s_{h+1}=f]\cdot V_{h+1}^*(f)\\
    &=-2\gamma\left[\Big\langle v(a_1),v(a_2)\Big\rangle+2\gamma\right] + \left[\Big\langle v(a_1),v(a_2)\Big\rangle+2\gamma\right]\cdot \left( \Big\langle v(a_1), v(a^*) \Big\rangle + 2\gamma\right)\\
    &=\left(\Big\langle v(a_1),v(a_2)\Big\rangle+2\gamma\right)\cdot  \Big\langle v(a_1), v(a^*) \Big\rangle.
\end{align*}
This is exactly (\ref{eq:qstar}) for $h$. Hence both (\ref{eq:qstar}) and (\ref{eq:vstar}) hold for all $h\in [H]$.
\end{proof}

We now verify that the minimum suboptimality gap $\Delta_{\min}$ is lower bounded.
\begin{lemma}[Constant gap]
Assumption~\ref{assumption:gap} is satisfied with  $\Delta_{\min}= \frac{1}{24}$.
\end{lemma}
\begin{proof}
From Eq. (\ref{eq:qstar}) and (\ref{eq:vstar}), it is easy to see that at state $\overline{a_1}\neq \overline{a^*}$, for $a_2\neq a^*$, the suboptimality gap is
\begin{align*}
    \Delta_h(\overline{a_1},a_2) := V_h^*(\overline{a_1}) - Q^*_h(\overline{a_1}, a_2) \ge \gamma - \max\left\{ 3\gamma^2,\frac{3}{4}\gamma\right\} = \frac{1}{24}.
\end{align*}
Moreover, at state $\overline{a^*}$, for $a\neq a^*$, the suboptimality gap is
\begin{align*}
    \Delta_h(\overline{a^*}, a) := V_h^*(\overline{a^*}) - Q^*_h(\overline{a^*}, a) \ge \frac{3}{4}\gamma - 3\gamma^2= \frac{1}{24}.
\end{align*}
Finally, for the terminal state $f$, the suboptimality gap is obviously $1$. Therefore $\Delta_{\min}\ge \frac{1}{24}$ for all MDPs under consideration. 
\end{proof}


\subsection{The information-theoretic argument}
Now we are ready to state and prove our main technical lemma.
\begin{lemma}
\label{lem:lower_formal}
For any algorithm, there exists $a^*\in [m]$ such that in order to output $\pi$ with
\begin{equation*}
   \E_{s_1\sim\mu}V^{\pi}(s_1)\ge \E_{s_1\sim\mu}V^*(s_1)-0.05
\end{equation*}
with probability at least $0.1$ for $\mathcal{M}_{a^*}$, the number of samples required is $2^{\Omega(\min\{d,H\})}$.
\end{lemma}
We provide a proof sketch for the lower bound below. The full proof can be found in Appendix~\ref{sec:lowerproof}.
Our main result, Theorem~\ref{thm:lower}, is a direct consequence of Lemma~\ref{lem:lower_formal}. 
\paragraph{Proof sketch.} 
Observe that the feature map of $\mathcal{M}_{a^*}$ does not depend on $a^*$, and that for $h<H$ and $a_2\neq a^*$, the reward $R_h(\overline{a_1},a_2)$ also contains no information about $a^*$. The transition probabilities are also independent of $a^*$, unless the action $a^*$ is taken. Moreover, the reward at state $f$ is always $0$. Thus, to receive information about $a^*$, the agent either needs to take the action $a^*$, or be at a non-game-over state at the final time step ($h=H$).

However, note that the probability of remaining at a non-terminal state at the next layer is at most
$$\sup_{a_1\neq a_2}\langle v(a_1), v(a_2)\rangle+2\gamma \le 3\gamma\le \frac{3}{4}.$$
Thus for any algorithm, $\Pr[s_H\neq f]\le \left(\frac{3}{4}\right)^H$, which is exponentially small.

In other words, any algorithm that does not know $a^*$ either needs to ``be lucky'' so that $s_H=f$, or needs to take $a^*$ ``by accident''. Since the number of actions is $m=2^{\Theta(d)}$, either event cannot happen with constant probability unless the number of episodes is exponential in $\min\{d,H\}$. 

In order to make this claim rigorous, we can construct a reference MDP $\mathcal{M}_0$ as follows. The state space, action space, and features of $\mathcal{M}_0$ are the same as those of $\mathcal{M}_a$. The transitions are defined as follows:
    \begin{align*}
        \Pr[\cdot|\overline{a_1},a_2] &= \begin{cases}
        \overline{a_2}:\Big\langle v(a_1), v(a_2)\Big\rangle+2\gamma \\
        f:1-\Big\langle v(a_1), v(a_2)\Big\rangle-2\gamma
        \end{cases},\tag*{($\forall a_1, a_2$ s.t. $a_1\neq a_2$)}\\
    \Pr[f|f,\cdot] &= 1.
    \end{align*}
The rewards are defined as follows:
\begin{align*}
     R_h(\overline{a_1},a_2) &:= -2\gamma\left[\Big\langle v(a_1),v(a_2)\Big\rangle+2\gamma\right], \tag*{( $\forall a_1, a_2$ s.t. $a_1\neq a_2$)}\\
        R_h(f,\cdot) &:= 0.
\end{align*}
Note that $\mathcal{M}_0$ is identical to $\mathcal{M}_{a^*}$, except when $a^*$ is taken, or when an trajectory ends at a non-terminal state. Since the latter event happens with an exponentially small probability, we can show that for any algorithm the probability of taking $a^*$ in $\mathcal{M}_{a^*}$ is close to the probability of taking $a^*$ in $\mathcal{M}_0$. Since $\mathcal{M}_0$ is independent of $a^*$, unless an exponential number of samples are used, for any algorithm there exists $a^*\in[m]$ such that the probability of taking $a^*$ in $\mathcal{M}_0$ is $o(1)$. It then follows that the probability of taking $a^*$ in $\mathcal{M}_{a^*}$ is $o(1)$. Since $a^*$ is the optimal action for every state, such an algorithm cannot output a near-optimal policy for $\mathcal{M}_{a^*}$.

\section{Upper Bound}
\label{sec:upper}
Theorem~\ref{thm:lower} suggests that Assumption~\ref{assumption:realizability} and Assumption~\ref{assumption:gap} are not sufficient for sample-efficient RL when the number of actions could be exponential, and that additional assumptions are needed to achieve polynomial sample complexity. One style of assumption is via assuming a global representation property on the features, such as completeness~\citep{zanette2020learning}. 

In this section, we consider two assumptions on additional structures on the transitions of the MDP rather than the feature representation that enable good rates for linear regression with sparse bias. The first condition is a variant of the low variance condition in~\citet{du2019provably}. 
\begin{assumption}[Low variance condition]
\label{assumption:var}
There exists a constant $1\le \cvar< \infty$ such that for any $h\in[H]$ and any policy $\pi$, 
\begin{equation*}
    \E_{s\sim \D^\pi_h}\left[\left|V^\pi(s)-V^*(s)\right|^2\right]\le \cvar\cdot\left( \E_{s\sim \D^\pi_h}\left[\left|V^\pi(s)-V^*(s)\right|\right]\right)^2.
\end{equation*}
\end{assumption}

We also consider an alternative assumption where feature distribution is hypercontractive.
\begin{assumption}
\label{assumption:hypercontractive}
(Hypercontractivity of $\phi$)
There exists a constant $1\le \chyper<\infty$ such that for any $h\in [H]$ and any policy $\pi$, the distribution of $\phi(s,a)$ with $(s,a)\sim \D^\pi_h$ is $(\chyper,4)$-hypercontractive. In other words, $\forall \pi$, $\forall h\in [H]$, $\forall v\in\R^d$, 
\begin{equation*}
    \E_{(s,a)\sim \D^\pi_h}\left[(\phi(s,a)^\top v)^4\right]\le \chyper\cdot\left(\E_{(s,a)\sim \D^\pi_h}[(\phi(s,a)^\top v)^2]\right)^2.
\end{equation*}
\end{assumption}
Intuitively, hypercontractivity characterizes the anti-concentration of a distribution. A broad class of distributions are hypercontractive with $\chyper=O(1)$, including Gaussian distributions (of arbitrary covariance matrices), uniform distributions over the hypercube and sphere, and strongly log-concave distributions~\citep{kothari2017outlier}. Hypercontractivity has been previously used for outlier-robust linear regression~\citep{klivans2018efficient,bakshi2020robust} and moment-estimation~\citep{kothari2017outlier}.

We show that under Assumptions~\ref{assumption:realizability},~\ref{assumption:gap},~\ref{assumption:var} or ~\ref{assumption:realizability},~\ref{assumption:gap},~\ref{assumption:hypercontractive}, a modified version of the Difference Maximization Q-learning (DMQ) algorithm~\citep{du2019provably} is able to learn a near-optimal policy using polynomial number of trajectories with no dependency on the number of actions.

\subsection{Optimal experiment design}
\label{sec:gopt}
Given a set of $d$-dimensional vectors, (optimal) experiment design aims at finding a distribution $\rho$ over the vectors such that when sampling from this distribution, linear regression performs optimally based on some criteria. In this paper, we will use the G-optimality criterion, which minimizes the maximum prediction variance over the set. The following lemma on G-optimal design is a  direct corollary of the Kiefer-Wolfowitz theorem~\citep{kiefer1960equivalence}.
\begin{lemma}[Existence of G-optimal design]
\label{lemma:oldgopt}
For any set $X\subseteq \R^d$, there exists a distribution $\mathcal{K}_X$ supported on $X$, known as the G-optimal design, such that
\begin{equation*}
    \max_{x\in X}x^{\top}\left(\E_{z\sim \mathcal{K}_X}zz^{\top}\right)^{-1}x\le d.
\end{equation*}
\end{lemma}
Efficient algorithms for finding such a distribution can be found in~\citet{todd2016minimum}.

In the context of reinforcement learning, the set $X$ corresponds to the set of all features, which is inaccessible. Instead, one can only observe one state $s$ at a time, and choose $a\in\mathcal{A}$ based on the features $\{\phi(s,a)\}_{a\in\mathcal{A}}$. Such a problem is closer to the distributional optimal design problem described by~\citet{ruan2020linear}. For our purpose, the following simple approach suffices: given a state $s$, perform exploration by sampling from the G-optimal design on $\{\phi(s,a)\}_{a\in\mathcal{A}}$. The performance of this exploration strategy is guaranteed by the following lemma, which will be used in the analysis in Section~\ref{sec:upper}.
\begin{lemma}[Lemma 4 in~\citet{ruan2020linear}]
\label{lemma:goptdesign}
For any state $s$, denote the $G$-optimal design with its features by
$\rho_s(\cdot)\in \Delta_{A}$, and the corresponding covariance matrix by $\Sigma_s:=\sum_a \rho_s(a)\phi(s,a)\phi(s,a)^{\top}$. Given a distribution $\nu$ over states. Denote the average covariance matrix by $\Sigma:=\E_{s\sim \nu}\Sigma_s$. Then
$$\E_{s\sim \nu}\left[\max_{a\in\mathcal{A}}\phi(s,a)^{\top}\Sigma^{-1}\phi(s,a)\right]\le d^2.$$
\end{lemma}
We provide a proof in the appendix for completeness.

Note that the performance of this strategy is only worse by a factor of $d$ (compared to the case where one can query all features), and has no dependency on the number of actions.

\subsection{The modified DMQ algorithm}

\paragraph{Overview.} During the execution of the Difference Maximization Q-learning (DMQ) algorithm, for each level $h \in [H]$, we maintain three variables: the estimated linear coefficients $\theta_h \in \mathbb{R}^d$, a set of exploratory policies $\Pi_h$, and the empirical feature covariance matrix $\Sigma_h$ associated with $\Pi_h$. We initialize $\theta_h=\mathbf{0}\in\R^d$, $\Sigma_h:=\mathbf{0}_{d\times d}$ and $\Pi_h$ to as a single purely random exploration policy, i.e., $\Pi_h = \{\pi\}$ where $\pi$ chooses an action uniformly at random for all states.\footnote{We also define a special $\Pi_0$ in the same manner.}

Each time we execute Algorithm~\ref{algo:layer}, the goal is to update the estimated linear coefficients $\theta_h \in \mathbb{R}^d$, so that for all $\pi \in \Pi_h$, $\theta_h$ is a good estimation to $\theta_h^*$ with respect to the distribution induced by $\pi$.
We run ridge regression on the data distribution induced by policies in $\Pi_h$, and the regression targets are collected by invoking the greedy policy induced by $\{\theta_{h'}\}_{h' > h}$.

However, there are two apparent issues with such an approach. First, for levels $h ' > h$, $\theta_{h'}$ is guaranteed to achieve low estimation error only with respect to the distributions induced by policies $\Pi_{h'}$. 
It is possible that for some $\pi \in \Pi_h$, the estimation error of $\theta_{h'}$ is high for the distribution induced by $\pi$ (followed by the greedy policy). 
To resolve this issue, the main idea in~\citet{du2019provably} is to explicitly check whether $\theta_{h'}$ also predicts well on the new distribution (see Line~\ref{line:check} in Algorithm~\ref{algo:layer}). 
If not, we add the new policy into $\Pi_{h'}$ and invoke Algorithm~\ref{algo:layer} recursively.
The analysis in~\citet{du2019provably} upper bounds the total number of recursive calls by a potential function argument, which also gives an upper bound on the sample complexity of the algorithm. 

Second, the exploratory policies $\Pi_h$ only induce a distribution over states at level $h$, and the algorithm still needs to decide an exploration strategy to choose actions at level $h$. To this end, the algorithm in~\citet{du2019provably} explores all actions uniformly at random, and therefore the sample complexity has at least linear dependency on the number of actions. We note that similar issues also appear in the linear contextual bandit literature~\citep{lattimore2020bandit,ruan2020linear}, and indeed our solution here is to explore by sampling from the G-optimal design over the features at a single state. 
As shown by Lemma~\ref{lemma:goptdesign}, for all possible roll-in distributions, such an exploration strategy achieves a nice coverage over the feature space, and is therefore sufficient for eliminating the dependency on the size of the action space. 

\begin{algorithm2e}[hbt]
\SetKwInput{KwInput}{Input}
\SetKwInput{KwOutput}{Output}
\SetAlgoVlined
\LinesNumbered
 \KwInput{A level $h \in \{0,\cdots,H\}$}
 \For{$\pi_h\in \Pi_h$}{
     \For{$h'=H,H-1,\cdots,h+1$}{
        Collect $N$ samples $\{(s_{h'}^j,a_{h'}^j)\}_{j\in [N]}$ with $s_{h'}^j\sim \D^{\tilde\pi_h}_{h'}$ and $a_{h'}^j\sim \rho_{s_{h'}^j}$ ($\tilde{\pi}_h$ defined in~(\ref{eq:policy}))\\
        $\hat\Sigma_{h'}\gets \frac{1}{N}\sum_{j=1}^N  \phi(s_{h'}^j,a_{h'}^j)\phi(s_{h'}^j,a_{h'}^j)^\top$\\
        \If{$\Vert \Sigma_{h'}^{-\frac{1}{2}}\hat\Sigma_{h'} \Sigma_{h'}^{-\frac{1}{2}}\Vert_2>\beta |\Pi_{h'}|$}
        {\label{line:check}

            $\Pi_{h'}\gets \Pi_{h'}\cup \{\tilde{\pi}_h\}$\\
            Run LearnLevel recursively on level $h'$\\
            Goto Line $1$ (restart the algorithm)
        }
             }
 }
 \If{$h=0$}{
 Output greedy policy with respect to $\{\theta_h\}_{h\in[H]}$\\
 }
 $\Sigma_h\gets \frac{\lambda_r}{|\Pi_h|} I$,$\quad w_h \gets \mathbf{0} \in \mathbb{R}^d$\\
 \For{$i=1,\cdots,N\left|\Pi_h\right|$}{
    Sample $\pi$ from uniform distribution over $\Pi_h$\\
    Execute $\tilde\pi_h$ (see~(\ref{eq:policy})) to collect $(s^i_h,a^i_h,y_i)$, where $y_i:=\sum_{h'\ge h}r^i_{h'}$ is the on-the-go reward\\
    $\Sigma_h\gets \Sigma_h + \frac{1}{N|\Pi_h|}\phi(s^i_h,a^i_h)\phi(s^i_h,a^i_h)^{\top}$\\
    $w_h \gets w_h + \frac{1}{N|\Pi_h|}\phi(s^i_h,a^i_h)y_i$
 }
 $\theta_h \gets \left((\lridge- \frac{\lambda_r}{|\Pi_h|})I+\Sigma_h\right)^{-1} w_h$\\
 \caption{LearnLevel}
  \label{algo:layer}
\end{algorithm2e}

\paragraph{The algorithm.}
The formal description of the algorithm is given in Algorithm~\ref{algo:layer}. The algorithm should be run by calling LearnLevel on input $h=0$. 

Here, for a policy $\pi_h\in\Pi_h$, the associated exploratory policy $\Tilde{\pi}_h$ is defined as 
\begin{equation}
\label{eq:policy}
     \Tilde{\pi}_h(s_{h'})=\begin{cases}
    \pi(s_{h'}) & (\text{if } h'<h)\\
    \text{Sample from }\rho_{s_h}(\cdot) & (\text{if }h'=h)\\
    \arg\max_{a}\phi_{h'}(s_{h'},a)^{\top}\theta_{h'} & (\text{if } h'>h)\\
    \end{cases}.
\end{equation}
Here $\rho_s(\cdot)$ is the G-optimal design on the set of vectors $\{\phi(s,\cdot)\}_{a\in\mathcal{A}}$, as defined by Lemma~\ref{lemma:oldgopt}. Note that when $h=0$, $\tilde\pi_h$ is always the greedy policy on $\{\theta_{h}\}_{h\in[H]}$. The choice of the algorithmic parameters ($\beta$, $\lambda_r$, $\lridge$) can be found in the proof of Theorem~\ref{thm:dsec}.

\subsection{Analysis}
We show the following theorem regarding the modified algorithm.
\begin{theorem}
\label{thm:dsec}
Assume that Assumption~\ref{assumption:realizability},~\ref{assumption:gap} and one of Assumption~\ref{assumption:var} and~\ref{assumption:hypercontractive} hold. Also assume that
\begin{align*}
    \epsilon&\le \poly(\Delta_{\min},1/\cvar,1/d,1/H) \tag*{(Under Assumption~\ref{assumption:var})}\\
    {\rm or} \quad \epsilon&\le \poly(\Delta_{\min},1/\chyper,1/d,1/H) \tag*{(Under Assumption~\ref{assumption:hypercontractive})}.
\end{align*}
Let $\mu$ be the initial state distribution. Then with probability $1-\epsilon$, running Algorithm~\ref{algo:layer} on input $0$ returns a policy $\pi$ which satisfies $\E_{s_1\sim\mu}V^{\pi}(s_1)\ge \E_{s_1\sim\mu}V^*(s_1)-\epsilon$ using $\poly(1/\epsilon)$ trajectories.
\end{theorem}
Note that here both the algorithm and the theorem have no dependence on $A$, the number of actions. The proof of the theorem under Assumption~\ref{assumption:var} is largely based on the analysis in~\cite{du2019provably}. The largest difference is that we used Lemma~\ref{lemma:goptdesign} instead of the original union bound argument when controlling $\Pr\left[\sup_{a}|\theta_h^\top\phi(s,a) Q_h^*(s,a)|>\frac{\gamma}{2}\right]$. The proof under Assumption~\ref{assumption:hypercontractive} relies on a novel analysis of least squares regression under hypercontractivity. The full proof is deferred to Appendix~\ref{sec:dsecproof}.

\section{Discussion}
\label{sec:discussion}

\paragraph{Exponential separation between the generative model and the online setting.} When a generative model (also known as simulator) is available, Assumption~\ref{assumption:realizability} and Assumption~\ref{assumption:gap} are sufficient for designing an algorithm with $\poly(1/\epsilon,1/\Delta_{\min},d,H)$ sample complexity~\citep[Theorem C.1]{du2019good}. As shown by Theorem~\ref{thm:lower}, under the standard online RL setting (i.e. without access to a generative model), the sample complexity is lower bounded by $2^{\Omega(\min\{d,H\})}$ when $\Delta_{\min}=\Theta(1)$, under the same set of assumptions. This implies that the generative model is \emph{exponentially} more powerful than the standard online RL setting. 

Although the generative model is conceptually much stronger than the online RL model, previously little is known on the extent to which the former is more powerful. In tabular RL, for instance, the known sample complexity bounds with or without access to generative models are nearly the same~\citep{zhang2020almost,agarwal2020model}, and both match the lower bound (up to logarithmic factors). 
To the best of our knowledge, the only existing example of such separation is shown by~\citet{wang2020reward} under the following set of conditions: (i) deterministic system; (ii) realizability (Assumption~\ref{assumption:realizability}); (iii) no reward feedback (a.k.a. reward-free exploration). In comparison, our separation result holds under less restrictions (allows stochasticity) and for the usual RL environment (instead of reward-free exploration), and it therefore far more natural. 

\paragraph{Connecting Theorem~\ref{thm:lower} and Theorem~\ref{thm:dsec}.} 
Our hardness result in Theorem~\ref{thm:lower} shows that under Assumption~\ref{assumption:realizability} and Assumption~\ref{assumption:gap}, any algorithm requires exponential number of samples to find a near-optimal policy, and therefore, sample-efficient RL is impossible without further assumptions (e.g., Assumption~\ref{assumption:var} or Assumption~\ref{assumption:hypercontractive} assumed in Theorem~\ref{thm:dsec}). 
Indeed, Theorem~\ref{thm:lower} and Theorem~\ref{thm:dsec} imply that the coefficient $C$ in Assumption~\ref{assumption:var} and  $\chyper$ in Assumption~\ref{assumption:hypercontractive} is at least exponential  for the hard MDP family used in Theorem~\ref{thm:lower}.
In fact, the subtlety here is that $C$ in Assumption~\ref{assumption:var} and $\chyper$ in Assumption~\ref{assumption:hypercontractive} need to be upper bounded for all policies $\pi$.
It can be easily verified that for the hard MDP family used in Theorem~\ref{thm:lower}, for some policy $\pi$, $C$ in Assumption~\ref{assumption:var} and  $\chyper$ in Assumption~\ref{assumption:hypercontractive} is exponential. 

\paragraph{Minimum reaching probability.} We would also like to mention the following reachability condition, which is assumed by~\citet{du2019decode,misra2020kinematic}. Denote the probability under policy $\pi$ of visiting state $s$ at step $h$ by $\Pr^{\pi}_h[s]$; then it is assumed that $\min_{s\in\mathcal{S},h>1}\max_{\pi}\Pr^{\pi}_h[s]=\eta_{\min}>0$. Although in $\mathcal{M}_{a^*}$, $\eta_{\min}\le 2^{-\Theta(H)}$ is exponentially small, one can slightly modify the construction so that $\eta_{\min}=1$ (see Appendix~\ref{sec:minreach} for details). The rationale is straightforward: in the $\mathcal{M}_{a^*}$, the action $a^*$ will not be taken with high probability in polynomial samples. 
Thus, one can exploit this by setting $a^*$ to lead to a special state from which all states are reachable.
This suggests that even under Assumption~\ref{assumption:realizability}, Assumption~\ref{assumption:gap} and the condition that $\eta_{\min}=1$, there is an exponential sample complexity lower bound.

\paragraph{Open problems.}
The first open problem, perhaps an obvious one, is whether a sample complexity lower bound under Assumption~\ref{assumption:realizability} can be shown with polynomial number of actions. This will further rule out $\poly(A,d,H)$-style upper bounds, which are still possible with the current results. An even stronger lower bound would be one with polynomial number of actions under both realizability and minimum suboptimality gap. Such lower bound will likely settle the sample complexity of reinforcement learning with linear function approximation under the realizability assumption. 
Another open problem is whether Assumption~\ref{assumption:var} and~\ref{assumption:hypercontractive} can be replaced by or understood as more natural characterizations of the complexity of the MDP.



\section*{Acknowledgments}
The authors would like to thank Kefan Dong and Dean Foster for helpful discussions.
Sham M. Kakade acknowledges funding from the ONR award N00014-18-1-2247. 
RW  was supported in part by the NSF IIS1763562, US Army W911NF1920104, and ONR Grant N000141812861.

\bibliographystyle{plainnat}
\bibliography{ref}

\newpage
\appendix

\section{Proof of Lemma~\ref{lemma:goptdesign}}
\begin{proof}
We state a proof of this lemma for completeness. By Lemma~\ref{lemma:oldgopt}, $\forall s$, 
\begin{equation*}
   \max_{a\in\mathcal{A}}\phi(s,a)^{\top}\Sigma_s^{-1}\phi(s,a)\le d. 
\end{equation*}
It follows that $\forall a\in\mathcal{A}$,
\begin{equation*}
\phi(s,a)\phi(s,a)^\top\preccurlyeq d\Sigma_s. 
\end{equation*}
Therefore,
\begin{align*}
\E_{s\sim \nu}\left[\max_{a\in\mathcal{A}}\phi(s,a)^{\top}\Sigma^{-1}\phi(s,a)\right]&=\E_{s\sim\nu}\max_{a\in\mathcal{A}}{\rm Tr}\left(\phi(s,a)\phi(s,a)^\top \Sigma^{-1}\right)\\
&\le \E_{s\sim\nu}{\rm Tr}\left(d\Sigma_s\Sigma^{-1}\right)=d^2.
\end{align*}
\end{proof}

\section{Proof of Theorem~\ref{thm:lower}}
\label{sec:lowerproof}
\begin{proof}
We consider $K$ episodes of interaction between the algorithm and the MDP $\mathcal{M}_{a}$. Since each trajectory is a sequence of $H$ states, we define the total number of samples as $KH$. Denote the state, the action and the reward at episode $k$ and timestep $h$ by $s^k_h$, $a^k_h$ and $r^k_h$ respectively.

Consider the following reference MDP denoted by $\mathcal{M}_0$. The state space, action space, and features of this MDP are the same as those of the MDP family. The transitions are defined as follows:
    \begin{align*}
        \Pr[\cdot|\overline{a_1},a_2] &= \begin{cases}
        \overline{a_2}:\Big\langle v(a_1), v(a_2)\Big\rangle+2\gamma \\
        f:1-\Big\langle v(a_1), v(a_2)\Big\rangle-2\gamma
        \end{cases},\tag*{($\forall a_1, a_2$ s.t. $a_1\neq a_2$)}\\
    \Pr[f|f,\cdot] &= 1.
    \end{align*}
The rewards are defined as follows:
\begin{align*}
     R_h(\overline{a_1},a_2) &:= -2\gamma\left[\Big\langle v(a_1),v(a_2)\Big\rangle+2\gamma\right], \tag*{( $\forall a_1, a_2$ s.t. $a_1\neq a_2$)}\\
        R_h(f,\cdot) &:= 0.
\end{align*}
Intuitively, this MDP is very similar to the MDP family, except that the optimal action $a^*$ is removed. More specifically, $\mathcal{M}_0$ is identical to $\mathcal{M}_a$ except when the action $a$ is taken at a non-terminal state, or when an episode ends at a non-terminal state.

More specifically, we claim that for $t<H$, $\forall s_t,a_t$ such that $a_t\neq a$,
\begin{equation*}
    \pr_{\mathcal{M}_a}[s_{t+1}|s_t,a_t] =  \pr_{\mathcal{M}_0}[s_{t+1}|s_t,a_t],
\end{equation*}
and that for $t<H$, $\forall s_t,a_t$ such that $a_t\neq a$, 
\begin{equation*}
    r_{t}^{\mathcal{M}_a}(s_t,a_t) = r_{t}^{\mathcal{M}_0}(s_t,a_t).
\end{equation*}
Also, $r_{H}^{\mathcal{M}_a}(s_t,a_t) = r_{H}^{\mathcal{M}_0}(s_t,a_t)$ if $s_t=f$. It follows that
\begin{align*}
&\pr_{\mathcal{M}_a}\left[s^1_1,a^1_1,r^1_1,\cdots s^k_h, a^k_h, r^k_h\left|a\notin A^k_h, \forall k'\le k, s^{k'}_H=f\right.\right]\\
=&\pr_{\mathcal{M}_0}\left[s^1_1,a^1_1,r^1_1,\cdots s^k_h, a^k_h, r^k_h\left|a\notin A^k_h, \forall k'\le k, s^{k'}_H=f\right.\right].
\end{align*}
Here $A^k_h$ is a shorthand for $\left\{a^1_1,a^1_2,\cdots,a^1_H,\cdots, a^k_h\right\}$, i.e. all actions taken up to timestep $h$ for episode $k$. By marginalizing the states and the actions, we get
\begin{align*}
\pr_{\mathcal{M}_a}\left[a^k_h\left|a\notin A^k_h, \forall k'\le k, s^{k'}_H=f\right.\right]=\pr_{\mathcal{M}_0}\left[a^k_h\left|a\notin A^k_h, \forall k'\le k, s^{k'}_H=f\right.\right].
\end{align*}
It then follows that
\begin{align*}
\pr_{\mathcal{M}_a}\left[a^k_h=a\left|a\notin A^k_h, \forall k'\le k, s^{k'}_H=f\right.\right]=\pr_{\mathcal{M}_0}\left[a^k_h=a\left|a\notin A^k_h, \forall k'\le k, s^{k'}_H=f\right.\right].
\end{align*}
Next, we prove via induction that
\begin{equation}
\label{eq:probinduction}
        \pr_{\mathcal{M}_a}\left[a\in A^k_h\left|\forall k'\le k, s^{k'}_H=f\right.\right]=\pr_{\mathcal{M}_0}\left[a\in A^k_h\left|\forall k'\le k, s^{k'}_H=f\right.\right].
\end{equation}
Suppose that (\ref{eq:probinduction}) holds up to $(k,h-1)$. Then
\begin{align*}
&\pr_{\mathcal{M}_a}\left[a\in A^k_h\left|\forall k'\le k, s^{k'}_H=f\right.\right]\\
=&  \pr_{\mathcal{M}_a}\left[a\notin A^k_{h-1}\right]\pr_{\mathcal{M}_a}\left[a^k_h=a\left|a\notin A^k_{h-1},\forall k'\le k, s^{k'}_H=f\right.\right]+\pr_{\mathcal{M}_a}\left[a\in A^k_{h-1}\left|\forall k'\le k, s^{k'}_H=f\right.\right]\\
=& \pr_{\mathcal{M}_0}\left[a\notin A^k_{h-1}\right]\pr_{\mathcal{M}_0}\left[a^k_h=a\left|a\notin A^k_{h-1},\forall k'\le k, s^{k'}_H=f\right.\right]+\pr_{\mathcal{M}_0}\left[a\in A^k_{h-1}\left|\forall k'\le k, s^{k'}_H=f\right.\right]\\
=& \pr_{\mathcal{M}_0}\left[a\in A^k_h\left|\forall k'\le k, s^{k'}_H=f\right.\right].
\end{align*}
That is, (\ref{eq:probinduction}) holds for $h$, $k$ as well. By induction, (\ref{eq:probinduction}) holds for all $h$, $k$. Thus,
\begin{align*}
\pr_{\mathcal{M}_a}\left[a\in A^k_h\right]&\le \pr_{\mathcal{M}_a}\left[a\in A^k_h\left|\forall k'\le k, s^{k'}_H= f\right.\right] + \Pr\left[\exists k'\le k, s^{k'}_H\neq f\right]\\
&\le \pr_{\mathcal{M}_0}\left[a\in A^k_h\left|\forall k'\le k, s^{k'}_H= f\right.\right] + k\cdot\left(\frac{3}{4}\right)^H.
\end{align*}
Since $|A^k_h|\le kH$, $\sum_{a\in [m]}\pr_{\mathcal{M}_0}\left[a\in A^k_h\left|\forall k'\le k, s^{k'}_H= f\right.\right]\le kH$. It follows that there exists $a^*\in [m]$ such that
$$\pr_{\mathcal{M}_0}\left[a^*\in A^K_H\left|\forall k'\le K, s^{k'}_H= f\right.\right]\le \frac{KH}{m}.$$
As a result
\begin{align*}
\pr_{\mathcal{M}_{a^*}}\left[a^*\in A^K_H\right]\le  \frac{KH}{m}+K\left(\frac{3}{4}\right)^H.
\end{align*}
Recall that $m=\lfloor \exp(\frac{1}{8}\gamma^2d)\rfloor$ and $\gamma = \frac16$. 
Therefore, unless $KH=2^{\Omega\left(\min\{d,H\}\right)}$, the probability of taking the optimal action $a^*$ in the interaction with $\mathcal{M}_{a^*}$ is $o(1)$. 

From the suboptimality gap condition, it follows that if $\E_{s_1\sim\mu}V^{\pi}(s_1)\ge \E_{s_1\sim\mu}V^*(s_1)-0.05$, $\Pr\left[a_1\neq a^*\land s_1\neq \overline{a^*}\right]\cdot \Delta_{\min}\le 0.05$. Hence
$$\Pr\left[a_1=a^*\right]\ge 1-\left(0.8+\frac{1}{m}\right)=0.2-\frac{1}{m}.$$
Therefore, if the algorithm is able to output such a policy with probability $0.1$, it is able to take the action $a^*$ in the next episode with $\Theta(1)$ probability by executing $\pi$. However, as proved above, this is impossible unless $KH=2^{\Omega\left(\min\{d,H\}\right)}$.
\end{proof}

\section{Ensuring Minimum Reachability}
\label{sec:minreach}

As mentioned in Section~\ref{sec:discussion}, the hardness results holds even when all states can be reached with probability $1$. 
In this section, we demonstrate how to modify the construction in Section~\ref{sec:lower} so that all states can be reached with probability $1$.

Based on the construction in Section~\ref{sec:lower}, we will modify $\mathcal{M}_{a^*}$ in the following manner.
\begin{enumerate}
    \item Remove the state $\overline{a^*}$.
    \item Create a deterministic initial state $s$; at state $s$, taking action $a\neq a^*$ leads to state $\overline{a}$; taking a special action $0$ leads to the terminal state $f$; taking action $a^*$ leads to state $s$ again.
    \item Set $\phi(s,0)=(-1,\mathbf{0}$); $\phi(s,a)=(-1,v(a))$.
    \item $R(s,a)=-1-2\gamma$ for $a\in[m]\neq a^*$. $R(s,a^*)=0$. $R(s,0)=-1$.
\end{enumerate}
It can be seen that now, $\eta_{\min}=1$ because to reach state $\overline{a}$ at level $h$, one can simply take the action sequence $a^*,\cdots, a^*,a$. Realizability at the initial state $s$ can also be verified (with $\theta^*=(1,v(a^*))$ as before).

\section{Proof of Theorem~\ref{thm:dsec}}
\label{sec:dsecproof}
\begin{proof}[Proof under Assumption~\ref{assumption:var}]
Let us set $\beta=8$, $\lridge=\epsilon^2$, $\lambda_r=\epsilon^6$, $B=2d\log(\frac{d}{\lambda_r})$, $\epsilon_1=\epsilon^2$, 
$\epsilon_2=\frac{\lambda_r}{2B}$, 
$N=\frac{d\cdot\mathrm{log}(1/\epsilon_2)}{\epsilon_2^2}$. Recall that $\epsilon\le \text{poly}(\Delta_{\min},1/\cvar,1/d,1/H)$.
First, by Lemma~\ref{lemma:prob}, the event $\Omega$ holds with probability $1-\epsilon$; we will condition on this event in the following proof. By lemma~\ref{lemma:policysize}, when the algorithm terminates, $|\Pi_h|\le B$ for all $h\in [H]$. Note that the this implies that Algorithm~\ref{algo:layer} is called or restarted at most $H\cdot(1+B)$ times. In each call or restart of Algorithm~\ref{algo:layer}, at most $NB+N$ trajectories are sampled. Therefore, when the algorithm terminates, at most
$$H(1+B)\cdot(NB+N)\le \poly\left(1/\epsilon\right)$$
trajectories are sampled.

It remains to show that the greedy policy with respect to $\theta_1,\cdots,\theta_H$ is indeed $\epsilon$-optimal with high probability. To that end, let us state the following claims about the algorithm.
\begin{enumerate}
    \item Each time Line 9 is reached in Algorithm~\ref{algo:layer}, $\forall \pi\in\Pi_h$, define $\tilde\pi_h$ as in~(\ref{eq:policy}), $\forall h'>h$,
    \begin{equation}
        \label{eq:claim1}
        \E_{s_{h'}\sim \D^{\tilde\pi_{h}}_{h'}}\left[\sup_{a\in\mathcal{A}} \left|\phi(s_{h'},a)^\top (\theta_h-\theta^*_h)\right|^2\right] \le \frac{\Delta_{\min}^2\epsilon}{4H}.
    \end{equation} 
    \item Each time when $\theta_h$ is updated at Line 17, $\forall \pi\in\Pi_h$, define the associated covariance matrix at step $h$ as $\Sigma_h^\pi=\E_{s_h\sim\D_h^\pi,a_h\sim \rho_{s_h}}\left[\phi(s_h,a_h)\phi(s_h,a_h)^\top\right]$. Then $\Vert \theta_h-\theta^*_h\Vert^2_{\Sigma_h^\pi}\le 6B\cvar\epsilon^2$. It follows that
    \begin{equation}
        \label{eq:claim2}
 \E_{s_{h}\sim \D^{\pi}_{h}}\left[\sup_{a\in\mathcal{A}} \left|\phi(s_{h},a)^\top (\theta_h-\theta^*_h)\right|^2\right] \le \frac{\Delta_{\min}^2\epsilon}{4H}.
    \end{equation}
\end{enumerate}
Note that by the first claim with $h=0$, it follows that for the greedy policy $\hat\pi$ ($\tilde\pi_0$ is always the greedy policy) w.r.t. $\{\theta_h\}_{h\in [H]}$, $\forall h\in [H]$,
\begin{equation*}
    \E_{s_h\sim \D^{\hat\pi}_h}\left[\sup_{a\in\mathcal{A}} \left|\phi(s_{h},a)^\top (\theta_h-\theta^*_h)\right|^2\right] \le \frac{\Delta_{\min}^2\epsilon}{4H}.
\end{equation*}
Consequently by Markov's inequality,
\begin{align*}
    \Pr_{s_h\sim \D^{\hat\pi}_h}\left[\exists a\in\mathcal{A}:\left|\phi(s_{h},a)^\top (\theta_h-\theta^*_h)\right|>\frac{\Delta_{\min}}{2}\right]\le \frac{\epsilon}{H}.
\end{align*}
By Assumption~\ref{assumption:gap} and the fact that $\hat\pi$ takes the greedy action w.r.t. $\theta_h$, this implies that 
\begin{align*}
    \Pr_{s_h\sim \D^{\hat\pi}_h}\left[\hat\pi_h(s_h)\neq\pi^*_h(s_h)\right]\le \frac{\epsilon}{H}.
\end{align*}
Thus for a random trajectory induced by $\hat\pi$, with probability at least $1-\epsilon$, $\hat\pi_h(s_h)=\pi^*_h(s_h)$ for all $h=1,\cdots,H$, which proves the theorem.

It remains to prove the two claims.

\paragraph{Proof of (\ref{eq:claim2}).} We first prove the second claim based on the assumption that the first claim holds when Line $9$ is reached in the same execution of LearnLevel. By the first claim and the same arguments above, $\forall\pi\in\Pi_h$, construct $\tilde\pi_h$ as in (\ref{eq:policy}), then $\Pr_{s_{h'}\sim\D^{\tilde\pi_h}_{h'}}\left[\tilde\pi_{h}(s_{h'})\neq \pi^*(s_{h'})\right]\le \epsilon/H$. Thus, 
\begin{align*}
 \E_{s_{h+1}\sim\D^{\tilde\pi_h}_{h+1}}\left[V^{\tilde\pi_h}_{h+1}(s_{h+1})\right]\ge \E_{s_{h+1}\sim\D^{\tilde\pi_h}_{h+1}}\left[V^{*}_{h+1}(s_{h+1})\right]-\epsilon.
\end{align*}
By Assumption~\ref{assumption:var}, this suggests that
\begin{align*}
    \E_{s_{h+1}\sim\D^{\tilde\pi_h}_{h+1}}\left[\left(V^{\tilde\pi_h}_{h+1}(s_{h+1})-V_{h+1}^*(s_{h+1})\right)^2\right]\le \cvar\epsilon^2.
\end{align*}
When $(s_h,a_h,y)$ is sampled,
\begin{align*}
    \E\left[y|s_h,a_h\right]&= \E\left[R(s_h,a_h)+V^{\tilde\pi_h}_{h+1}(s_{h+1})|s_h,a_h\right]\\ &=Q^*(s_h,a_h)+\E\left[V^{\tilde\pi_h}_{h+1}(s_{h+1})-V_{h+1}^*(s_{h+1})|s_h,a_h\right],
\end{align*}
where the expectation is over trajectories induced by $\tilde\pi_h$. In other words, $y_i:=\sum_{h'\ge h}r_h^i$ can be written as $\phi(s_h^i,a_h^i)^\top\theta^*_h+b_i+\xi_i$, where $\xi_i$ is mean-zero independent noise with $|\xi_i|\le 2$ almost surely and $b_i:=\sum_{h'>h}r_{h'}^i-V^*_{h+1}(s^i_{h+1})$ satisfies $\E[b_i^2]\le \cvar\epsilon^2$. Note that $\theta_h$ is the ridge regression estimator for this linear model. By Lemma~\ref{lemma:ridge},
\begin{align*}
    \E_{\pi\sim\text{Unif}(\Pi_h), s_h\sim\D^{\pi}_h,a_h\sim\rho_{s_h}}\left[\left|\phi(s_h,a_h)^\top(\theta_h-\theta^*_h)\right|^2\right]&\le 4(\cvar\epsilon^2+\epsilon_1+\lridge)\le 6\cvar\epsilon^2.
\end{align*}
It follows that $\forall \pi\in\Pi_h$,
\begin{align*}
    \E_{ s_h\sim\D^{\pi}_h,a_h\sim\rho_{s_h}}\left[\left|\phi(s_h,a_h)^\top(\theta_h-\theta^*_h)\right|^2\right]&\le |\Pi_h|\cdot 6\cvar\epsilon^2\le 6B\cvar\epsilon^2.
\end{align*}
Now, by Lemma~\ref{lemma:goptdesign}, 
\begin{align*}
     &\E_{s_{h}\sim \D^{\pi}_{h}}\left[\sup_{a\in\mathcal{A}} \left|\phi(s_{h},a)^\top (\theta_h-\theta^*_h)\right|^2\right]\\
     \le& \E_{s_{h}\sim \D^{\pi}_{h}}\left[\sup_{a\in\mathcal{A}} \Vert\phi(s_{h},a)\Vert^2_{(\Sigma_h^{\pi})^{-1}}\right]\cdot \Vert\phi_h-\phi^*_h\Vert^2_{\Sigma_h^\pi}\\
     \le& d^2\cdot 6B\cvar\epsilon^2\le \frac{\Delta_{\min}^2\epsilon}{4H}.
\end{align*}
This proves the second claim.

\paragraph{Proof of (\ref{eq:claim1}).} Now, let us prove the first claim, assuming that the second claim holds for the last update of any $\theta_h$. By observing Algorithm~\ref{algo:layer}, if Line $9$ is reached, during the last execution of the first \texttt{for} loop (i.e. Lines $1$ to $8$), the \texttt{if} clause at Line $5$ must have returned False every time (otherwise the algorithm will restart). It follows that during the last execution of Lines $1$ to $8$, neither $\{\theta_h\}_{h\in [H]}$ nor $\{\Pi_h\}_{h\in [H]}$ is updated.

Consider the \texttt{if} clause when checking $\pi\in\Pi_h$ for layer $h'$. Recall that $$\Sigma_{h'}^{\tilde\pi_h}=\E_{s_{h'}\sim\D_h^{\tilde\pi_h},a_{h'}\sim \rho_{s_{h'}}}\left[\phi(s_{h'},a_{h'})\phi(s_{h'},a_{h'})^\top\right].$$ Also define $\Sigma^*_{h'}:=\frac{\lambda_r}{|\Pi_{h'}|}I+\E_{\pi\sim\text{Unif}(\Pi_{h'})}\Sigma_{h'}^{\pi}$. Then by Lemma~\ref{lemma:dsec},
\begin{align*}
\Vert \left(\Sigma^*_{h'}\right)^{-\frac{1}{2}}\Sigma_{h'}^{\tilde\pi_h}\left(\Sigma^*_{h'}\right)^{-\frac{1}{2}}\Vert_2 \le 3\beta |\Pi_{h'}|.
\end{align*}
It follows that
\begin{align*}
\Vert \theta_h-\theta^*_h\Vert^2_{\Sigma_{h'}^{\tilde\pi_h}} &= (\theta_h-\theta^*_h)^\top \Sigma_{h'}^{\tilde\pi_h} (\theta_h-\theta^*_h)\\
&=\left((\Sigma^{*}_{h'})^{\frac{1}{2}}(\theta_h-\theta^*_h)\right)^{\top}\left(\left(\Sigma^*_{h'}\right)^{-\frac{1}{2}}\Sigma_{h'}^{\tilde\pi_h}\left(\Sigma^*_{h'}\right)^{-\frac{1}{2}}\right)\left((\Sigma^{*}_{h'})^{\frac{1}{2}}(\theta_h-\theta^*_h)\right)\\
&\le \Vert \theta_h-\theta^*_h\Vert^2_{\Sigma^*_{h'}}\cdot \Vert \left(\Sigma^*_{h'}\right)^{-\frac{1}{2}}\Sigma_{h'}^{\tilde\pi_h}\left(\Sigma^*_{h'}\right)^{-\frac{1}{2}}\Vert_2\\
&\le 3\beta B\cdot\left(\lambda_r\cdot\left(\frac{2}{\lridge}\right)^2+6B\cvar\epsilon^2\right)\\
&\le 24B^2\cdot 10\cvar\epsilon^2.
\end{align*}
By Lemma~\ref{lemma:goptdesign},
\begin{equation*}
\E_{s_{h'}\sim\D_h^{\tilde\pi_h}}\left[\sup_{a\in\mathcal{A}} 
\Vert \phi(s_{h'},a)\Vert^2_{(\Sigma_{h'}^{\tilde\pi_h})^{-1}}
\right]\le d^2.
\end{equation*}
As a result,
\begin{align*}
    \E_{s_{h'}\sim \D^{\tilde\pi_{h}}_{h'}}\left[\sup_{a\in\mathcal{A}} \left|\phi(s_{h'},a)^\top (\theta_h-\theta^*_h)\right|^2\right] &\le 
    \E_{s_{h'}\sim\D_h^{\tilde\pi_h}}\left[
     \Vert \theta_h-\theta^*_h\Vert^2_{\Sigma_{h'}^{\tilde\pi_h}}\cdot
    \sup_{a\in\mathcal{A}} \Vert \phi(s_{h'},a)\Vert^2_{(\Sigma_{h'}^{\tilde\pi_h})^{-1}}
    \right]\\
    &\le 240B^2\cvar\epsilon^2\cdot d^2\le \frac{\epsilon\Delta_{\min}^2}{4H}.
\end{align*}
This proves the first claim. The failure probability of the algorithm is controlled by Lemma~\ref{lemma:prob}.
\end{proof}

\begin{proof}[Proof under Assumption~\ref{assumption:hypercontractive}]
The proof under Assumption~\ref{assumption:hypercontractive} is quite similar, except that we will use Lemma~\ref{lemma:ridgehyper} instead of Lemma~\ref{lemma:ridge} for the analysis of ridge regression.

Let us set $\beta=8$, $\epsilon_0=\epsilon^2$, $\lridge=\epsilon^3$, $\lambda_r=\epsilon^9$, $B=2d\log(\frac{d}{\lambda_r})$, $\epsilon_1=\epsilon^3$, 
$\epsilon_2=\frac{\lambda_r}{2B}$, 
$N=\frac{d}{\epsilon_2^3}$. Recall that $\epsilon\le \poly(\Delta_{\min},1/\chyper, 1/d,1/H)$. We will state similar claims about the algorithm.

\begin{enumerate}
    \item Each time Line 9 is reached in Algorithm~\ref{algo:layer}, $\forall \pi\in\Pi_h$, define $\tilde\pi_h$ as in~(\ref{eq:policy}), $\forall h'>h$,
    \begin{equation}
        \label{eq:claim1h}
        \E_{s_{h'}\sim \D^{\tilde\pi_{h}}_{h'}}\left[\sup_{a\in\mathcal{A}} \left|\phi(s_{h'},a)^\top (\theta_h-\theta^*_h)\right|^2\right] \le \frac{\Delta_{\min}^2\epsilon_0}{4H}.
    \end{equation} 
    \item Each time when $\theta_h$ is updated at Line 17, $\forall \pi\in\Pi_h$, define the associated covariance matrix at step $h$ as $\Sigma_h^\pi=\E_{s_h\sim\D_h^\pi,a_h\sim \rho_{s_h}}\left[\phi(s_h,a_h)\phi(s_h,a_h)^\top\right]$. Then $\Vert \theta_h-\theta^*_h\Vert^2_{\Sigma_h^\pi}\le \frac{\Delta_{\min}^2\epsilon_0}{120HBd^2}$. It follows that
    \begin{equation}
        \label{eq:claim2h}
 \E_{s_{h}\sim \D^{\pi}_{h}}\left[\sup_{a\in\mathcal{A}} \left|\phi(s_{h},a)^\top (\theta_h-\theta^*_h)\right|^2\right] \le \frac{\Delta_{\min}^2\epsilon_0}{4H}.
    \end{equation}
\end{enumerate}
We now prove the two claims in similar fashion.

\paragraph{Proof of (\ref{eq:claim2}).} We first prove the second claim based on the assumption that the first claim holds when Line $9$ is reached in the same execution of LearnLevel. By the first claim, $\forall\pi\in\Pi_h$, construct $\tilde\pi_h$ as in (\ref{eq:policy}), then
\begin{equation}
    \label{eq:policycorrect}
    \Pr_{s_{h'}\sim\D^{\tilde\pi_h}_{h'}}\left[\tilde\pi_{h}(s_{h'})\neq \pi^*(s_{h'})\right]\le \epsilon_0/H.
\end{equation}
When $(s_h,a_h,y)$ is sampled,
\begin{align*}
    \E\left[y|s_h,a_h\right]&= \E\left[R(s_h,a_h)+V^{\tilde\pi_h}_{h+1}(s_{h+1})|s_h,a_h\right]\\ &=Q^*(s_h,a_h)+\E\left[V^{\tilde\pi_h}_{h+1}(s_{h+1})-V_{h+1}^*(s_{h+1})|s_h,a_h\right],
\end{align*}
where the expectation is over trajectories induced by $\tilde\pi_h$. In other words, $y_i:=\sum_{h'\ge h}r_h^i$ can be written as $\phi(s_h^i,a_h^i)^\top\theta^*_h+b_i+\xi_i$, where $\xi_i$ is mean-zero independent noise with $|\xi_i|\le 2$ almost surely, and $b_i$ is defined as
$$b_i:=-\sum_{h'>h}\left(V^*(s_{h'}^i)-Q^*(s_{h'}^i,a_{h'}^i)\right).$$
Here $\E[\xi_i]=0$ because $$\E[\xi_i]=\E\left[\sum_{h'\ge h}r_{h'}^i\right]-Q^*_h(s^i_h,a^i_h)-\E[b_i]=Q^{\tilde\pi_h}(s_h^i,a_h^i)-Q^*(s_{h}^i,a_h^i)+\left(Q^*(s_{h}^i,a_h^i)-Q^{\tilde\pi_h}(s_h^i,a_h^i)\right)=0.$$
By (\ref{eq:policycorrect}), $\Pr[b_i\neq 0]\le \epsilon_0$. Thus by Lemma~\ref{lemma:ridgehyper}, 
\begin{align*}
    \E_{\pi\sim\text{Unif}(\Pi_h), s_h\sim\D^{\pi}_h,a_h\sim\rho_{s_h}}\left[\left|\phi(s_h,a_h)^\top(\theta_h-\theta^*_h)\right|^2\right]&\le 8\left(\epsilon_1+\lridge\right)+288\epsilon_0^{1.5}\chyper^{2.5}d^{4.5}\left(\frac{2B}{\epsilon}\right)^{0.5}\\
    &\le 16\epsilon^3+288\epsilon^{2.5}\chyper^{2.5}d^{4.5}(2B)^{0.5}.
\end{align*}
It follows that $\forall \pi\in\Pi_h$,
\begin{align*}
    \E_{ s_h\sim\D^{\pi}_h,a_h\sim\rho_{s_h}}\left[\left|\phi(s_h,a_h)^\top(\theta_h-\theta^*_h)\right|^2\right]&\le |\Pi_h|\cdot \left(16\epsilon^2+288\epsilon^{2.5}\chyper^{2.5}d^{4.5}(2B)^{0.5}\right)\le \frac{\Delta_{\min}^2\epsilon_0}{120HBd^2},
\end{align*}
where we used the fact $\epsilon\le \poly(\Delta_{\min}, 1/\chyper, 1/d, 1/H)$. Now, by Lemma~\ref{lemma:goptdesign}, 
\begin{align*}
     \E_{s_{h}\sim \D^{\pi}_{h}}\left[\sup_{a\in\mathcal{A}} \left|\phi(s_{h},a)^\top (\theta_h-\theta^*_h)\right|^2\right]
     \le& \E_{s_{h}\sim \D^{\pi}_{h}}\left[\sup_{a\in\mathcal{A}} \Vert\phi(s_{h},a)\Vert^2_{(\Sigma_h^{\pi})^{-1}}\right]\cdot \Vert\phi_h-\phi^*_h\Vert^2_{\Sigma_h^\pi}\\
     \le& d^2\cdot \frac{\Delta_{\min}^2\epsilon_0}{120HBd^2}\le \frac{\Delta_{\min}^2\epsilon_0}{4H}.
\end{align*}
This proves the second claim.

\paragraph{Proof of (\ref{eq:claim1h}).} Now, let us prove the first claim, assuming that the second claim holds for the last update of any $\theta_h$. Consider Line $9$ when checking for $\pi\in\Pi_h$ for layer $h'$. Recall that $$\Sigma_{h'}^{\tilde\pi_h}=\E_{s_{h'}\sim\D_h^{\tilde\pi_h},a_{h'}\sim \rho_{s_{h'}}}\left[\phi(s_{h'},a_{h'})\phi(s_{h'},a_{h'})^\top\right].$$
Similar to the proof under Assumption~\ref{assumption:var}, we can bound $\Vert \theta_h-\theta^*_h\Vert_{\Sigma_{h'}^{\tilde\pi_h}}$ by
\begin{align*}
\Vert \theta_h-\theta^*_h\Vert^2_{\Sigma_{h'}^{\tilde\pi_h}}&\le  \Vert \theta_h-\theta^*_h\Vert^2_{\Sigma^*_{h'}}\cdot \Vert \left(\Sigma^*_{h'}\right)^{-\frac{1}{2}}\Sigma_{h'}^{\tilde\pi_h}\left(\Sigma^*_{h'}\right)^{-\frac{1}{2}}\Vert_2\\
&\le 3\beta B\cdot\left(\lambda_r\cdot\left(\frac{2}{\lridge}\right)^2+\frac{\Delta_{\min}^2\epsilon_0}{120HBd^2}\right)\\
&\le 96B\epsilon^3+\frac{\Delta_{\min}\epsilon_0}{5Hd^2}. 
\end{align*}
By Lemma~\ref{lemma:goptdesign}, $
\E_{s_{h'}\sim\D_h^{\tilde\pi_h}}\left[\sup_{a\in\mathcal{A}} 
\Vert \phi(s_{h'},a)\Vert^2_{(\Sigma_{h'}^{\tilde\pi_h})^{-1}}
\right]\le d^2$. Consequently
\begin{align*}
    \E_{s_{h'}\sim \D^{\tilde\pi_{h}}_{h'}}\left[\sup_{a\in\mathcal{A}} \left|\phi(s_{h'},a)^\top (\theta_h-\theta^*_h)\right|^2\right] &\le 
    \E_{s_{h'}\sim\D_h^{\tilde\pi_h}}\left[
     \Vert \theta_h-\theta^*_h\Vert^2_{\Sigma_{h'}^{\tilde\pi_h}}\cdot
    \sup_{a\in\mathcal{A}} \Vert \phi(s_{h'},a)\Vert^2_{(\Sigma_{h'}^{\tilde\pi_h})^{-1}}
    \right]\\
    &\le 96B\epsilon^3d^2+\frac{\Delta_{\min}\epsilon_0}{5H}\le \frac{\Delta_{\min}\epsilon_0}{4H}.
\end{align*}
In the last inequality we used $\epsilon_0=\epsilon^2$ and $\epsilon\le \poly(\Delta_{\min},1/d,1/H)$. This proves (\ref{eq:claim1h}). Finally the failure probability is controlled in Lemma~\ref{lemma:prob}.
\end{proof}

\begin{lemma}[Covariance concentration~\citep{tropp2015introduction}]
\label{lemma:matrixchernoff}
Suppose $M_1,\cdots,M_N\in\R^{d\times d}$ are i.i.d. random matrices drawn from a distribution $\D$ over positive semi-definite matrices. If $\Vert M_t\Vert_F\le 1$ almost surely and $N=\Omega\left(\frac{d\log(d/\delta)}{\epsilon^2}\right)$, then with probability $1-\delta$,
\begin{equation*}
    \left\Vert \frac{1}{N}\sum_{i=1}^N M_t-\E_{M\sim\D}[M]\right\Vert_2 \le \epsilon.
\end{equation*}
\end{lemma}

\begin{lemma}[Risk bound for ridge regression, Lemma A.2~\citet{du2019provably}]
\label{lemma:ridge}
Suppose that $(x_1,y_1)$, $\cdots$, $(x_N,y_N)$ are i.i.d. data drawn from $\D$ with 
$$y_i=\theta^\top x_i+b_i+\xi_i,$$
where $\E_{(x_i,y_i)\sim \D}[b_i^2]\le \eta$, $|\xi_i|\le 2n$ almost surely and $\E[\xi_i]=0$. Let the ridge regression estimator be
\begin{equation*}
\hat\theta=\left(\sum_{i=1}^N x_ix_i^\top + N\lridge\cdot I\right)^{-1}\cdot\sum_{i=1}^N x_iy_i.
\end{equation*}
If $N=\Omega\left(\frac{d}{\epsilon_N^2}\log(\frac{d}{\delta})\right)$, then with probability at least $1-\delta$,
\begin{equation*}
    \E_{x\sim\D}\left[\left((\hat\theta-\theta)^\top x\right)^2\right] \le 4\left(\eta+\epsilon_N+\lridge\right).
\end{equation*}
\end{lemma}

\begin{lemma}[Failure probability]
\label{lemma:prob}
Define the following events regarding the execution of Algorithm~\ref{algo:layer}.
\begin{enumerate}
    \item $\Omega_1$: Each time $\Sigma_h$ is updated,
    \begin{equation}
    \label{eq:condition1}
        \left\Vert \Sigma_h-\E_{\pi\sim\text{Unif}(\Pi_h),s_h\sim\D^{\pi}_h,a_h\sim\rho_{s_h}}\left[\phi(s_h,a_h)\phi(s_h,a_h)^\top\right]\right\Vert_2\le \epsilon_2.
    \end{equation}
    \item $\Omega_2$: Each time $\theta_h$ is updated,
    \begin{equation}
    \label{eq:condition2}
    \E_{\pi\sim\text{Unif}(\Pi_h),s\sim\D^{\pi}_h,a\sim\rho_s}\left[\left((\theta_h-\theta^*_h)^\top \phi(s,a)\right)^2\right] \le 4\left(\eta+\epsilon_1+\lridge\right),
    \end{equation}
    where $\eta$ is defined as in Lemma~\ref{lemma:ridge}.
    \item $\Omega_3$: Each time $\theta_h$ is updated,
    \begin{equation}
    \label{eq:condition3}
    \E_{\pi\sim\text{Unif}(\Pi_h),s\sim\D^{\pi}_h,a\sim\rho_s}\left[\left((\theta_h-\theta^*_h)^\top \phi(s,a)\right)^2\right] \le 288\eta^{1.5}C^{2.5}d^{4.5}\left(\frac{2B}{\epsilon}\right)^{0.5},
    \end{equation}
    where $\eta$ and $C$ are defined as in Lemma~\ref{lemma:ridgehyper}.
\end{enumerate}
Then under Assumption~\ref{assumption:var}, $\Pr[\Omega_1\cap \Omega_2]\ge 1-\epsilon$. Alternatively, under Assumption~\ref{assumption:hypercontractive},  $\Pr[\Omega_1\cap \Omega_3]\ge 1-\epsilon$.
\end{lemma}
\begin{proof}
Note that $N\ge\frac{d\log(1/\epsilon_2)}{\epsilon_2^2}$ where $\epsilon_2\le \frac{\epsilon^6}{d}$. Therefore, by Lemma~\ref{lemma:matrixchernoff}, each time $\Sigma_h$ is updated, (\ref{eq:condition1}) holds with probability at least $1-\epsilon^2$.

As for (\ref{eq:condition2}), note that $N\ge\frac{d\log(1/\epsilon_2)}{\epsilon_2^2}\gg \frac{d}{\epsilon_1^2}\cdot\log(\frac{d}{\epsilon^2})$. Thus by Lemma~\ref{lemma:ridge}, each time $\theta_h$ is updated, (\ref{eq:condition2}) holds with probability at least $1-\epsilon^2$.

Similarly, for (\ref{eq:condition3}), under the choice of parameters under Assumption~\ref{assumption:hypercontractive}, 
$N\ge \frac{d}{\epsilon_2^3}\gg \left(\frac{d}{\epsilon_2^2}+\frac{1}{\eta}\right)\ln\frac{2dB}{\epsilon}+\frac{2B}{\epsilon}$. Thus by Lemma~\ref{lemma:ridgehyper}, the probability that (\ref{eq:condition3}) is violated each step is at most $\epsilon/2B$.

Note that when the algorithm terminates, the $\Sigma_h$ and $\theta_h$ are updated at most $|\Pi_h|$ times. Also note that, if during the first $B$ updates, neither (\ref{eq:condition1}) nor (\ref{eq:condition2}) are violated, by Lemma~\ref{lemma:policysize} it follows that $|\Pi_h|\le B$ when the algorithm terminates. In other words,
\begin{equation*}
    \Pr[\Omega_1\cup \Omega_2]\ge 1-B\cdot 2\epsilon^2\ge 1-\epsilon.
\end{equation*}
Similarly, under Assumption~\ref{assumption:hypercontractive},
\begin{equation*}
    \Pr[\Omega_1\cup \Omega_3]\ge 1-B\cdot \epsilon^2-B\cdot\frac{\epsilon}{2B}\ge 1-\epsilon.
\end{equation*}
\end{proof}
\begin{lemma}[Distribution shift error checking]
\label{lemma:dsec}
Assume that $\epsilon_2<\min\{\frac{1}{2}\beta\lambda_r,\frac{\lambda_r}{2B}\}$. Consider the \texttt{if} clause when checking for $\pi_h\in\Pi_h$, i.e. when computing $\Vert \Sigma_{h'}^{-\frac{1}{2}}\hat\Sigma_{h'}\Sigma_{h'}^{-\frac{1}{2}}\Vert_2$. Define
\begin{equation*}
M_{1}:=\frac{\lambda_r}{N|\Pi_{h'}|}I+\E_{\pi\sim\text{Unif}(\Pi_{h'}),s_{h'}\sim\D^{\pi}_{h'},a_{h'}\sim \rho_{s_{h'}}}\left[\phi(s_{h'},a_{h'})\phi(s_{h'},a_{h'})^\top\right],    
\end{equation*}
and
\begin{equation*}
M_2:=\E_{s_{h'}\sim\D^{\tilde\pi_h}_{h'}, a_{h'}\sim \rho_{s_{h'}}}\left[\phi(s_{h'},a_{h'})\phi(s_{h'},a_{h'})^\top\right].
\end{equation*}
Then under the event $\Omega$ defined in Lemma~\ref{lemma:prob}, when $\Vert \Sigma_{h'}^{-\frac{1}{2}}\hat\Sigma_{h'}\Sigma_{h'}^{-\frac{1}{2}}\Vert_2\le \beta|\Pi_{h'}|$,
\begin{equation*}
\Vert M_{1}^{-1/2}M_2M_1^{-1/2}\Vert_2\le 3\beta|\Pi_{h'}|.
\end{equation*}
When $\Vert \Sigma_{h'}^{-\frac{1}{2}}\hat\Sigma_{h'}\Sigma_{h'}^{-\frac{1}{2}}\Vert_2\ge \beta|\Pi_{h'}|$,
\begin{equation*}
\Vert M_{1}^{-1/2}M_2M_1^{-1/2}\Vert_2\ge \frac{1}{4}\beta|\Pi_{h'}|.
\end{equation*}
\end{lemma}
\begin{proof}
By Lemma~\ref{lemma:matrixchernoff},
\begin{equation*}
    \Vert M_1-\Sigma_{h'}\Vert_2 \le \epsilon_2\le\frac{\lambda_r}{2B}\le\frac{1}{2}\lambda_{\min}\left(\Sigma_{h'}\right).
\end{equation*}
Thus $\frac{1}{2}\Sigma_{h'} \preccurlyeq M_1\preccurlyeq 2\Sigma_{h'}$. Also by Lemma~\ref{lemma:matrixchernoff}, $\Vert M_2-\hat\Sigma_{h'}\Vert_2\le \epsilon_2$. Therefore, if $\Vert \Sigma_{h'}^{-\frac{1}{2}}\hat\Sigma_{h'}\Sigma_{h'}^{-\frac{1}{2}}\Vert_2\ge \beta|\Pi_{h'}|$,
\begin{align*}
\Vert M_{1}^{-1/2}M_2M_1^{-1/2}\Vert_2&\ge \frac{1}{2}\Vert \Sigma_{h'}^{-1/2}M_2\Sigma_{h'}^{-1/2}\Vert_2\ge \frac{1}{2}\Vert \Sigma_{h'}^{-1/2}\hat\Sigma_{h'}\Sigma_{h'}^{-1/2}\Vert_2-\frac{1}{2}\epsilon_2\Vert\Sigma_{h'}^{-1}\Vert_2\\
&\ge \frac{1}{2}\beta|\Pi_{h'}|-\frac{1}{2}\epsilon_2\cdot\frac{|\Pi_{h'}|}{\lambda_r}\ge \frac{1}{4}\beta|\Pi_{h'}|.
\end{align*}
Similarly, when $\Vert \Sigma_{h'}^{-\frac{1}{2}}\hat\Sigma_{h'}\Sigma_{h'}^{-\frac{1}{2}}\Vert_2\le \beta|\Pi_{h'}|$,
\begin{align*}
\Vert M_{1}^{-1/2}M_2M_1^{-1/2}\Vert_2&\le 2\Vert \Sigma_{h'}^{-1/2}M_2\Sigma_{h'}^{-1/2}\Vert_2\le 2\Vert \Sigma_{h'}^{-1/2}\hat\Sigma_{h'}\Sigma_{h'}^{-1/2}\Vert_2+2\epsilon_2\Vert\Sigma_{h'}^{-1}\Vert_2\\
&\le 2\beta|\Pi_{h'}|+2\epsilon_2\cdot\frac{|\Pi_{h'}|}{\lambda_r}\le 3\beta|\Pi_{h'}|.
\end{align*}
\end{proof}

\begin{lemma}[Lemma A.6 in~\citet{du2019provably}]
\label{lemma:policysize}
Under the event $\Omega_1$ defined in Lemma~\ref{lemma:prob} $|\Pi_h|\le B$ for all $h\in [H]$.
\end{lemma}

\section{Analysis of Ridge Regression under Hypercontractivity}
Recall that a distribution $\D$ is $(C,4)$-hypercontractive if $\forall v$,
\begin{align*}
    \E_{x\sim \D}[(x^\top v)^4]\le C\cdot \left(\E_{x\sim \D}[(x^\top v)^2]\right)^2.
\end{align*}
In this section we prove an strengthened version of Lemma~\ref{lemma:ridge} for hypercontractive distributions (Lemma~\ref{lemma:ridgehyper}), which may be of independent interest.

\begin{lemma}
\label{lemma:markov}
Let $x$ be a $d$-dimensional r.v. If the distribution of $x$ is $(C,4)$-hypercontractive and isotropic (i.e. $\E[xx^\top]=I$), then
$$\Pr[\Vert x\Vert_2>t]\le \frac{Cd^2}{t^4}.$$
\end{lemma}
\begin{proof}
Consider a Gaussian random vector $v\sim N(0,I)$. Then
$$\E_v[(x^\top v)^4]=\Vert x\Vert^4 \cdot \E_{\xi \sim N(0,1)}\xi^4 = 3\Vert x\Vert^4.$$
Therefore
\begin{align*}
    \E_{x}[\Vert x\Vert^4]&=\frac{1}{3}\E_{x,v}[(x^\top v)^4]\le \frac{C}{3}\E_v\left(\E_{x}(x^\top v)^2\right)^2\\
    &\le \frac{C}{3}\E_v\Vert v\Vert^4 = \frac{C\cdot (d^2+2d)}{3}\le d^2C.
\end{align*}
The claim then follows from Markov's inequality.
\end{proof}

\begin{lemma}
\label{lemma:hypertail}
If the $x_1,\cdots,x_n$ are i.i.d. samples from a $(C,4)$-hypercontractive distribution. Let $\sigma(\cdot)$ denote the decreasing order of $\Vert x_i\Vert_2$. Then with probability $1-\delta$, 
\begin{equation*}
    \sum_{k=1}^m \Vert x_{\sigma(k)}\Vert_2=3\delta^{-1/4}n^{1/4}m^{3/4}C^{1/4}d^{1/2}.
\end{equation*}
\end{lemma}
\begin{proof}
Fix $k\in [m]$. Set $t=\alpha\left(\frac{Cd^2n}{k}\right)^{1/4}$. By Lemma~\ref{lemma:markov}, 
\begin{align*}
    \Pr[\Vert x_{\sigma(k)}\Vert_2>t]&\le \binom{n}{k}\Pr[\Vert x\Vert>t]^k \le\binom{n}{k}\cdot\left(\frac{Cd^2}{t^4}\right)^k\\
		&\le\frac{n^k}{k!}\cdot\frac{k^k}{\alpha^{4k}n^k}\le\left(\frac{e}{\alpha^4}\right)^k.
\end{align*}
Choosing $\alpha=\left(\frac{2e}{\delta}\right)^{1/4}$ gives $\Pr[\Vert x_{\sigma(k)}\Vert_2>t]\le (\delta/2)^k$. By a union bound, with probability $1-\delta$,
\begin{equation*}
    \sum_{i=1}^m \Vert x_{\sigma(i)}\Vert_2 \le \sum_{k=1}^m (2e/\delta)^{1/4}\left(\frac{Cd^2n}{k}\right)^{1/4}\le  3\delta^{-1/4}n^{1/4}m^{3/4}C^{1/4}d^{1/2}.
\end{equation*}

\end{proof}

\begin{lemma}[Lemma 3.4~\citet{bakshi2020robust}]
If $\D$ is $(C,4)$-hypercontractive and $x_1,\cdots,x_n$ are i.i.d. samples drawn from $\D$. Let $\Sigma:=\E_{x\sim\D}[xx^\top]$. With probability $1-\delta$,
\begin{equation*}
    \left(1-\frac{Cd^2}{\sqrt{n\delta}}\right)\Sigma \preccurlyeq\frac{1}{n}\sum_{i=1}^n x_i x_i^\top\preccurlyeq \left(1+\frac{Cd^2}{\sqrt{n\delta}}\right)\Sigma.
\end{equation*}
\end{lemma}

\begin{lemma}[Risk bound for ridge regression with hypercontractivity]
\label{lemma:ridgehyper}
Suppose that $(x_1,y_1)$, $\cdots$, $(x_N,y_N)$ are i.i.d. data drawn from $\D$ with 
$$y_i=\theta^\top x_i+b_i+\xi_i,$$
where $\Pr[ b_i\neq 0]\le \eta $, $\Vert b\Vert_\infty\le 1$, $|\xi_i|\le 1$, and $\E[\xi_i]=0$. Assume that distribution of $x$ is $(C,4)$-hypercontractive (see Assumption~\ref{assumption:hypercontractive}). Let the ridge regression estimator be
\begin{equation*}
\hat\theta=\left(\sum_{i=1}^N x_ix_i^\top + N\lridge\cdot I\right)^{-1}\cdot\sum_{i=1}^N x_iy_i.
\end{equation*}
If $N=\Omega\left((\frac{d}{\epsilon_N^2}+\frac{1}{\eta})\log(\frac{d}{\delta})+\frac{1}{\delta}\right)$, then with probability at least $1-\delta$,
\begin{equation*}
    \E_{x\sim\D}\left[\left((\hat\theta-\theta)^\top x\right)^2\right] \le 8\left(\epsilon_N+\lridge\right)+288\eta^{1.5}C^{2.5}d^{4.5}\delta^{-0.5}.
\end{equation*}
\end{lemma}
\begin{proof}
Define $\hat\Sigma:=\frac{1}{N}\sum_{i=1}^N x_ix_i^\top$ and $\Sigma:=\E_{x\sim\D}[xx^\top]$. Then
\begin{align*}
    \hat\theta &= \frac{1}{N}\left(\lridge I+\hat\Sigma\right)^{-1}\sum_{i=1}^N \left(x_ix_i^\top \theta+x_i\cdot \xi_i+x_i\cdot b_i\right)\\
    &= \underbrace{\frac{1}{N}\left(\lridge I+\hat\Sigma\right)^{-1}\sum_{i=1}^N b_i x_i}_{(a)}+\underbrace{\frac{1}{N}\left(\lridge I+\hat\Sigma\right)^{-1}\sum_{i=1}^N\left(x_ix_i^\top \theta+x_i\cdot \xi_i\right)}_{(b)}.
\end{align*}
By Lemma~\ref{lemma:ridge}, $\Vert \theta - (b)\Vert_{\Sigma}^2 \le 4(\epsilon_N+\lridge)$. It remains to bound the $\Vert \cdot\Vert_{\Sigma}$ norm of $(a)$. 

First, by Hoeffding's inequality, with probability $1-\delta$, $\Vert b\Vert_0=\sum_{i=1}^n I[b_i\neq 0]\le 2\eta N$. Define $z_i:=\Sigma^{-1/2}x_i$ to be the normalized input. It can be seen that $\E[z_iz_i^\top]=I$ and that the distribution of $z_i$ is also hypercontractive. By Lemma~\ref{lemma:hypertail}, with probability $1-2\delta$,
\begin{align*}
    \sum_{i=1}^n \Vert z_i\Vert_2\cdot I[b_i\neq 0] \le 3\delta^{-1/4}N^{1/4}(2\eta N)^{3/4}(Cd^2)^{1/4}.
\end{align*}
It follows that with probability $1-2\delta$,
\begin{align*}
    \Vert(a)\Vert_\Sigma=\frac{1}{N}\left\Vert \hat\Sigma^{-1}\sum_{i=1}^N x_i b_i\right\Vert_{\Sigma}&\le  \frac{1}{N}\sum_{i=1}^N\Vert\Sigma^{1/2}{\hat\Sigma}^{-1} x_i b_i\Vert_2\\
	&= \frac{1}{N}\sum_{i=1}^N\Vert\Sigma^{1/2}\hat\Sigma^{-1}\Sigma^{1/2} z_i b_i\Vert_2\\
	&\le  \frac{1}{N}\Vert\Sigma^{1/2}\hat\Sigma^{-1}\Sigma^{1/2}\Vert_2 \cdot  \sum_{i=1}^N \Vert z_i\Vert_2\cdot H\cdot I[b_i\neq 0] \\
	&\le 3H\left(1+\frac{Cd^2}{\sqrt{N\delta}}\right)\cdot  \delta^{-1/4}N^{-3/4}(2\eta N)^{3/4}(Cd^2)^{1/4}\\
	&\le 12H\eta^{0.75}\cdot C^{\frac{5}{4}}d^{\frac{9}{4}}\delta^{-\frac{1}{4}}.
\end{align*}
Therefore
\begin{align*}
\Vert \hat\theta-\theta\Vert_{\Sigma}^2 &\le  2\Vert \hat\theta-(b)\Vert_{\Sigma}^2+2\Vert (a)\Vert_{\Sigma}^2\\
&\le 8(\epsilon_N+\lridge)+288\eta^{1.5}C^{2.5}d^{4.5}\delta^{-0.5}.
\end{align*}
\end{proof}

\end{document}